\documentclass[letterpaper]{article} 
\usepackage{aaai24}  
\usepackage{times}  
\usepackage{helvet}  
\usepackage{courier}  
\usepackage[hyphens]{url}  
\usepackage{graphicx} 
\usepackage{natbib}  
\usepackage{caption} 
\usepackage{amsfonts}
\usepackage{amsmath}
\usepackage{amsthm}

\newtheorem{definition}{Definition}
\newtheorem{theorem}{Theorem}
\usepackage{multirow}
\usepackage{adjustbox}
\usepackage{booktabs}
\usepackage{subcaption}
\usepackage{todonotes}
\usepackage{dsfont}

\usepackage{algorithm}
\usepackage{algorithmic}

\usepackage{newfloat}
\usepackage{listings}
\DeclareCaptionStyle{ruled}{labelfont=normalfont,labelsep=colon,strut=off} 
\floatstyle{ruled}
\newfloat{listing}{tb}{lst}{}
\floatname{listing}{Listing}
\nocopyright 

\title{Targeted Activation Penalties Help CNNs Ignore Spurious Signals}
\author {
    Dekai Zhang\textsuperscript{\rm 1},
    Matthew Williams\textsuperscript{\rm 2,\rm3},
    Francesca Toni\textsuperscript{\rm 1}
}
\affiliations {
    \textsuperscript{\rm 1}Department of Computing, Imperial College London \\
    \textsuperscript{\rm 2}Department of Radiotherapy, Charing Cross Hospital \\
    \textsuperscript{\rm 3}Institute of Global Health Innovation, Imperial College London
    \\
    \{dz819, matthew.williams, f.toni\}@imperial.ac.uk
}

\begin{document}

\maketitle

\begin{abstract}
Neural networks (NNs) can learn to rely on spurious signals in the training data, 
leading to poor generalisation. Recent methods 
tackle this problem by 
training NNs with additional
ground-truth annotations of such signals.
These methods 
may, however, let spurious signals re-emerge in deep 
convolutional NNs (CNNs).
We propose \emph{\textbf{T}argeted \textbf{A}ctivation \textbf{P}enalty (TAP)}, a new method tackling the same problem 
by penalising activations to control the re-emergence of spurious signals in deep CNNs,  
while also lowering training times and memory usage. In addition,  ground-truth annotations can be expensive to obtain. 
We show that TAP still works well with annotations generated by pre-trained models as effective substitutes of ground-truth annotations.
We demonstrate the power of TAP against two state-of-the-art baselines on the MNIST benchmark and on two clinical image datasets,
using four different CNN architectures.
\end{abstract}
\section{Introduction}
\label{sec:intro}
Neural networks (NNs) have demonstrated strong performances and in some domains have exceeded experts \citep{rajpurkarDeepLearningChest2018, keCheXtransferPerformanceParameter2021}. These success stories come with an important caveat: NNs appear to be very good at exploiting spurious signals which boost their performance on the training data but lead to poor generalisation \citep{hendricksWomenAlsoSnowboard2018}. \citet{ribeiroWhyShouldTrust2016a}, for instance, find that a model can achieve near perfect results in distinguishing huskies from wolves by making use of the background in the images. More worryingly, some models trained to detect pneumonia in chest radiographs \citep{zechVariableGeneralizationPerformance2018} were found to make heavy use of image artifacts. 

Recent work on \emph{explanatory supervision (XS)} has shown that models can be successfully protected from learning spurious signals by eliciting ground-truth explanations from humans as an additional supervision for the models \citep{rossRightRightReasons2017, tesoExplanatoryInteractiveMachine2019, riegerInterpretationsAreUseful2020, schramowskiMakingDeepNeural2020, hagosImpactFeedbackType2022, friedrichTypologyExploringMitigation2023}.
Of these methods, 
\emph{``right for the right reasons'' (RRR)}  \citep{rossRightRightReasons2017} and 
\emph{``right for better reasons'' (RBR)} \citep{shaoRightBetterReasons2021} seem to perform particularly well \citep{friedrichTypologyExploringMitigation2023}. These convincing performances, however, have been obtained on datasets requiring relatively shallow models (i.e., VGG-16 \citep{simonyanVeryDeepConvolutional2015} and shallower). It appears that they may be less successful with deeper convolutional NNs (CNNs), as illustrated under ``RRR'' and ``RBR'' in Figure~\ref{fig:explanations}.
Also, these works assume that extensive ground-truth explanations can be obtained, typically from humans, which in practice can be costly, especially in expert domains such as healthcare. 

In this paper, we define a novel XS method for CNNs, which adds a \emph{\textbf{T}argeted \textbf{A}ctivation \textbf{P}enalty (TAP)} to spurious signals, mitigating against the re-emergence of spurious signals in deeper layers (illustrated under ``TAP'' in Figure~\ref{fig:explanations}). We show that TAP performs competitively and better than RRR and RBR, 
while requiring lower training times and memory usage.
We further show that TAP can still perform well when replacing ground-truth explanations with noisier annotations generated by a teacher model pre-trained on as little as 1\% of the target domain. Our contributions can be summarised as follows\footnote{Source code: \texttt{https://github.com/dkaizhang/TAP}}:
\begin{itemize}
    \item We introduce TAP to teach CNNs to
    ignore spurious signals. We formally relate TAP to RRR and RBR: whereas RRR and RBR directly target input gradients, TAP indirectly does so by minimising activations, avoiding expensive second-order derivatives. In our experiments, this results in circa 25\% of the training time and half the memory usage compared to RRR and RBR.
    \item We compare TAP to RRR and RBR in the standard setting of a human teacher who provides ground-truth annotations, and we demonstrate that TAP can still be effective with noisy but automatically generated explanations from a teacher model which is pre-trained on a small clean dataset.

\item Our findings are supported by experiments (i) on MNIST \citep{lecunMNISTDatabaseHandwritten1998} using  a simple two-layer CNN
as a standard benchmark and, to show the efficacy of TAP on higher-stakes real world datasets, (ii) on two clinical datasets for pneumonia \citep{kermanyIdentifyingMedicalDiagnoses2018} and osteoarthritis \citep{chenFullyAutomaticKnee2019} 
using three commonly used architectures: VGG-16, ResNet-18 \citep{heDeepResidualLearning2016} and DenseNet-121 \citep{huangDenselyConnectedConvolutional2017}.
\end{itemize}
This paper is an extended version (with Appendix) of the paper published with the same title at AAAI-2024.

\begin{figure*}[ht]
\centering
\includegraphics[width=0.9\textwidth]{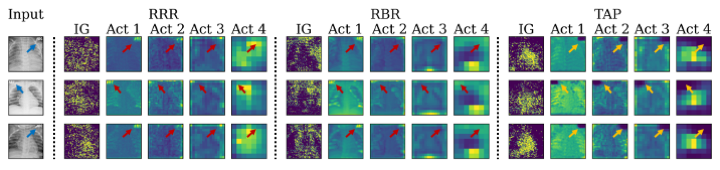} 
\caption{Three chest x-rays with spurious tags placed in the corners (blue arrows). Input gradients (IG) and activations at each of the four convolutional blocks (Act 1--4) of a ResNet-18 trained using RRR, RBR and TAP. With RRR and RBR the spurious 
tags can re-emerge in deeper layers 
(red arrows), while TAP mitigates their presence throughout (yellow arrows).}
\label{fig:explanations}
\end{figure*}

\section{Related Work}
\label{sec:background}

\paragraph{Explainable AI (XAI).}
NNs are frequently seen as black boxes which can be difficult to audit \citep{adlerAuditingBlackboxModels2018}. 
In computer vision, saliency maps are often used to highlight relevant areas in a given input \citep{simonyanDeepConvolutionalNetworks2014,springenbergStrivingSimplicityAll2015,sundararajanAxiomaticAttributionDeep2017,arrietaExplainableArtificialIntelligence2019,selvarajuGradCAMVisualExplanations2020} and have been successfully deployed to identify models which use spurious signals in the data \citep{ribeiroWhyShouldTrust2016a, lapuschkinUnmaskingCleverHans2019}. While XAI offers tools of identification, it typically does not address how to correct models that use spurious signals.

\paragraph{Explanatory Supervision (XS).}
Recent work on XS has investigated if ground-truth annotations of spurious signals can be used as an additional source of supervision to train models to make predictions for the ``right reasons'' \citep{rossRightRightReasons2017,tesoExplanatoryInteractiveMachine2019,riegerInterpretationsAreUseful2020,schramowskiMakingDeepNeural2020,shaoRightBetterReasons2021}. The pioneering method, RRR, by \citet{rossRightRightReasons2017} and RBR, a later extension of RRR by \citet{shaoRightBetterReasons2021}, use input gradient regularisation, which have been shown to be an effective mechanism to prevent models from learning spurious relationships \citep{friedrichTypologyExploringMitigation2023}. These, however, rely on expensive second-order derivatives
, whereas our proposed method does not. Furthermore, common to the above methods is the assumption of sample-wise ground-truth annotations, typically provided by a human. We find that, with our novel method, ground-truth annotations can in some cases be replaced with noisier annotations from a pre-trained teacher model.

\paragraph{Teacher-Student Settings.}
The setting we consider is closest to knowledge distillation \citep{hintonDistillingKnowledgeNeural2015} in which a student model receives 
supervision from a teacher model. 
Specifically, the student is trained to match logit targets provided by the teacher. Later extensions also match activations \citep{zagoruykoPayingMoreAttention2017} or Jacobians \citep{srinivasKnowledgeTransferJacobian2018}. The setting we consider differs in that the teacher does not directly provide additional targets to be matched but instead identifies areas to be ignored. We discuss these methods further in Appendix~\ref{appx:kd_comparison}.

\section{Preliminaries}
\label{sec:pre}

\paragraph{Setting.} Assume 
a \emph{dataset} $\{(X_i, y_i)\}_{i=1}^N$ of $N$ labelled images. For 
brevity, we assume that each image is single-channel, so that $X_i \in \mathcal{X} \subseteq \mathbb{R}^{H, W}$ consists of $H \times W$ pixels. Each label $y_i \in \mathcal{Y} \subseteq 
\{0,1\}^{K}$ is a one-hot vector over $K$ classes. Suppose we wish to learn, from this dataset, a CNN with $L$ convolutional layers, given by $f_{\theta}$,
parameterised by ${\theta}$,
such that for every \emph{input} $X \in \mathcal{X}$, 
$f_{\theta}(X) = \hat{y}$ with $\hat{y} 
\in \{0,1\}^K$ denoting the \emph{output vector}. 

Similar to \citet{adebayoPostHocExplanations2022}, we focus on a setting where the dataset 
is the result of a label-revealing \emph{contamination} of a \emph{clean} dataset $\{(X_i^*, y_i)\}_{i=1}^N$ where $X_i^* \in \mathcal{X}^* \subseteq \mathbb{R}^{H,W}$. The contaminated instance $X_i$ thus contains a signal which spuriously reveals $y_i$, whereas $X_i^*$ does not. 
Formally, we suppose there exists some \emph{spurious contamination function} $SC : \mathcal{X}^* \times \mathcal{Y} \rightarrow \mathcal{X}$ 
which induces spurious correlations in the clean dataset. Examples of spurious signals include medical tags, which are text strings frequently found on radiographs and which may spuriously correlate with the label \citep{adebayoPostHocExplanations2022}. The objective in this setting is to learn a classifier $f_{\theta}$ from the contaminated data that does not only classify well but which does so without relying on spurious signals. Intuitively, the classifier should produce the same output regardless of what spurious signals are added by the contamination function. 

\begin{definition}
\label{def:SC} Classifier $f_\theta$ is \emph{not reliant on spurious signals} $SC(X^*,y) \in \mathcal{X}$ for $X^* \in \mathcal{X}^*, y\in \mathcal{Y}$ if:
\begin{equation}
    f_{\theta}(SC(X^*,y)) = f_{\theta}(SC(X^*,permute(y)))
\label{eq:unbiased}
\end{equation}
where $permute(y)$ is some permutation of label $y$. 
\end{definition}

\begin{figure*}[ht]
\centering
\includegraphics[width=0.65\textwidth]{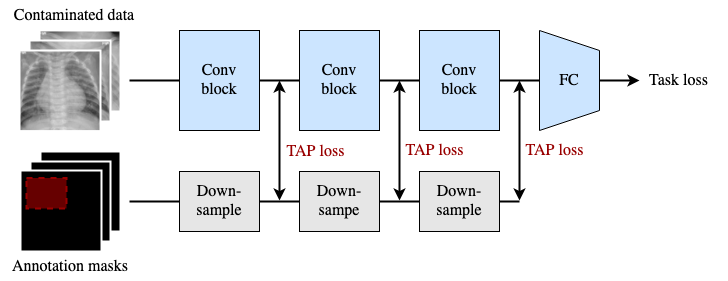} 
\caption{TAP losses target activations throughout the 
CNN with (input-level) annotation masks.}
\label{fig:tap}
\end{figure*}

\paragraph{Learning to Ignore Spurious Signals.}
If the model is making correct classifications  but is found to 
use 
spurious signals, it is not ``right for the right reasons'' \citep{rossRightRightReasons2017}. To revise the model, the common approach in XS is to obtain feedback, typically from a human, in the form of an annotation mask $M \in \mathbb{R}^{H,W}$ for each image, which indicates spurious signals 
therein.

To construct these masks, observe that each input $X\in \mathcal{X}$ 
can be decomposed into two disjoint matrices holding \emph{clean signals} $X^c$ and \emph{spurious signals} $X^s$, so that $X = X^c + X^s$. Note that $X^c$ therefore only contains the parts of the clean input $X^*$ unaffected by the transformation $SC$, so that, for $i=1,\ldots,H, 
j=1,\ldots,W$, $X^c_{ij} = X^*_{ij}$ if and only if $X^c_{ij} \neq 0$
. We then define the mask as an indicator of spurious signals:
\begin{equation}
    M_{ij} = \mathds{1}(X^c_{ij} = 0)
\end{equation} 

Given mask $M$, a common approach in XS is to augment the loss function $\mathcal{L}$ to supervise both the task and reasons:
\begin{equation}
\begin{aligned}
    \mathcal{L}(X, y, M; \theta) = \mathcal{L}_{Task}(X, y; \theta) + \lambda \mathcal{L}_{XS}(X, M; \theta) 
    \label{eq:lambda}
\end{aligned}                        
\end{equation}
where $\mathcal{L}_{Task}$ corresponds to the task (e.g., classification) loss, $\mathcal{L}_{XS}$ is the XS loss 
penalising the use of wrong reasons as defined by 
$M$, and $\lambda$ is a hyperparameter for the relative impact. Note the masks are not used to cover the input but define a penalty region, so that the model still ``sees" the entire input.

RRR is a pioneering instantiation of $\mathcal{L}_{XS}$, which targets spurious signals by penalising their input gradients \citep{rossRightRightReasons2017}:
\begin{equation}
\begin{aligned}
    \mathcal{L}^{RRR}_{XS}(X, M; \theta) = \sum_{i=1}^H \sum_{j=1}^W \left[ M \odot  \sum_{k=1}^K  \frac{\partial f_{\theta}(X)_{k}}{\partial X} \right]_{ij}^2
\end{aligned}                        
\label{eq:rrr}
\end{equation}
where $\odot$ denotes the Hadamard product. RBR, 
an extension of RRR, proposes to multiply influence functions with input gradients to take into account model changes from perturbing the input \citep{shaoRightBetterReasons2021}. Other instantiations exist \citep{riegerInterpretationsAreUseful2020,schramowskiMakingDeepNeural2020} but RRR and RBR appear to have the greatest corrective effect amongst XS methods \citep{friedrichTypologyExploringMitigation2023}. These methods, however, require higher-order derivatives which are expensive to compute. In this paper, we propose targeting activations of CNNs, which are computed as part of a single forward pass.

\section{Our Approach}
\label{sec:methods}
We introduce TAP, a new XS loss which targets activations of spurious signals with penalties (as illustrated in Figure~\ref{fig:tap}). We then frame the objective of teaching a model to ignore spurious signals as a teacher-student problem in which the teacher is a pre-trained model (see Figure~\ref{fig:framework}).

\subsection{TAP: Targeted Activation Penalty}
TAP takes advantage of the fact that convolutional layers preserve spatial relationships within an image in its activation map \citep{lecunDeepLearning2015}. Given this, we propose targeting parts of the activation map corresponding to spurious signals with penalties.

Suppose $Z^l \in \mathbb{R}^{C^l, H^l, W^l}$ is the tensor output of the $l$-th convolutional layer with $C^l$ output channels, filters $w^l$, bias $b^l$ and activation function $\sigma$, so that:
\begin{equation}
    Z^l = w^l * \sigma(Z^{l-1}) + b^l
\end{equation}
where the $*$-operator denotes the convolutional product. We define the \emph{activation map} $A^l$ 
as the channel-wise sum: 
\begin{equation}
    A^l = \sum_{c=1}^{C^l} \sigma(Z^l_c)
\end{equation}

Given an image $X \in \mathbb{R}^{H, W}$, where $H \geq H^l$ and $W \geq W^l$, we may need to downscale the annotation mask $M \in \mathbb{R}^{H, W}$  to match the dimensions of the activation map. We define a \emph{downscaling function} $\mathcal{D}: \mathbb{R}^{H, W} \times \mathbb{N} \rightarrow \mathbb{R}^{H^l, W^l}$, which accepts the input-level annotation mask $M$ and an integer to identify the layer and outputs a downscaled annotation mask $M^l$ with the target dimensions. To effectively target the areas of the spurious signals in the activations, the downscaling function needs to preserve the spatial information in the annotation masks and account for the receptive field of a given activation element. 
We apply max-pooling with a stride of 1 using a kernel size of $\kappa$ to effectively increase the annotation region by $\lfloor \kappa/2 \rfloor$ pixels in each direction. We then apply average pooling to match the height and width dimensions of layer $l$:
\begin{equation}    
    \mathcal{D}(M, l) = \text{AvgPool}(\text{MaxPool}(M), H^l, W^l) 
\end{equation}

We found a choice of $\kappa=3$ to be a good balance between capturing the influence of targeted pixels and preventing over-regularisation from too large a penalty region.

\begin{definition}
    Given a downscaling function $\mathcal{D}$, annotation mask $M$ and an $L$-layer CNN parameterised by $\theta$, producing activation maps $A^l$ in layer $l$, the \emph{Targeted Activation Penalty (TAP)} is:
        \begin{equation}
        \mathcal{L}^{TAP}_{XS}(X,M;\theta) = \sum_{l=1}^L \dfrac{1}{L} \left\lVert \mathcal{D}(M,l) \odot A^l \right\rVert  
    \end{equation}
    where $\left\lVert \cdot \right\rVert$ denotes the L1-norm. 
\end{definition}

In practice, for deep CNNs, instead of targeting the output of every convolutional layer, we focus on a subset (as envisaged in Figure~\ref{fig:tap}).    

\begin{figure*}[ht]
\centering
\includegraphics[width=0.6\textwidth]{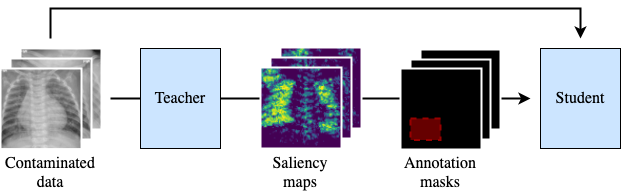} 
\caption{We use a pre-trained teacher model to derive saliency-based masks for 
a contaminated dataset to determine the least important regions in each input, resulting in an input-level mask for each input. The 
contaminated dataset and the corresponding masks are used to train a student model which learns to ignore the regions identified by the masks.}
\label{fig:framework}
\end{figure*}

\paragraph{Relationship with RRR and RBR.}
Both RRR and RBR directly reduce the input gradient of spurious signals to zero. We show in Theorem~\ref{th:relationship} that for CNNs with ReLU activation function $\sigma$ we can target elements in the activation map, as we do with TAP, to reduce the input gradient of pixel $(i,j)$ for some convolutional layer $l>1$, which is given by:
\begin{equation}
\begin{aligned}
    \frac{\partial}{\partial X_{ij}} \sigma(Z^l) 
    = \sigma'(Z^l) \odot w^l * \frac{\partial \sigma(Z^{l-1})}{\partial X_{ij}} 
\end{aligned}
\label{eq:input_grad_layer}
\end{equation}

\begin{theorem}
\label{th:relationship}
    Suppose we have a CNN with convolutional layers $l =1,\ldots,L$, filters $w^l$ of kernel size $\kappa^l$ and ReLU activation $\sigma$. It is sufficient to optimise all activation map elements $A^l_{ab}$ with respect to $w^l$ to set the input gradient of pixel $(i,j)$ in layer $l$ (Equation~\ref{eq:input_grad_layer}) to zero, where $(a,b)$ are such that:
    \begin{equation}
    \begin{aligned}
        i - \kappa^1 - \ldots - \kappa^l + l \leq a \leq i  \\
        j - \kappa^1 - \ldots - \kappa^l + l \leq b \leq j 
    \end{aligned}
    \label{eq:receptive}
    \end{equation}
\end{theorem}

We provide a sketch of the proof here and defer details to Appendix~\ref{appx:proof}. Note that activation elements $(a,b)$ outside the area defined by Equation~\ref{eq:receptive} do not capture pixel $(i,j)$ within their receptive fields, so that the input gradient of $(i,j)$ is zero for such elements. To ensure that it is zero for elements $(a,b)$ inside the area, we optimise activation $A^l_{ab}$ with respect to filters $w^l_c$, $c = 1,\ldots,C^l$.  The first-order condition sets the following to zero:
\begin{equation}
\begin{aligned}
    \frac{\partial A^l_{ab}}{\partial w^{l}_c} &= \sigma'(Z^l_{cab})
    \begin{bmatrix}
        \delta^{l-1}_{a,b} &\!\dots\!& \delta^{l-1}_{a+\kappa^l-1,b} \\
        \vdots &\!\ddots\!& \vdots \\
        \delta^{l-1}_{a,b+\kappa^l-1} &\!\dots\!& \delta^{l-1}_{a+\kappa^l-1,b+\kappa^l-1}     
    \end{bmatrix}
\end{aligned}
\end{equation}
where $\delta^{l-1} = \sigma(Z^{l-1})$ which is constant with respect to $w^l$. Given ReLU activation, the optimisation results in $\sigma'(Z^l_{ab})$ being pushed to zero. 
\\

Figure~\ref{fig:explanations} illustrates that TAP indeed mitigates the presence of spurious signals in the input gradients even when only a subset of the corresponding activations are targeted. We provide further quantitative evidence in Appendix~\ref{appx:quant_comparison}, showing that, with TAP, less than 1\% of input gradients in the top-quartile by magnitude overlap with spurious signals. Notably, TAP does not require computing second-order derivatives to do so, which results in faster training times and lower memory consumption compared to RRR and RBR (Table~\ref{tab:train_speed}).

\begin{table}[h]
\centering
\begin{tabular}{@{}lcccccc@{}}
\toprule
  & \multicolumn{2}{c}{VGG-16} & \multicolumn{2}{c}{ResNet-18} & \multicolumn{2}{c}{Dense-121} \\ \cmidrule(lr){2-3}\cmidrule(lr){4-5}\cmidrule(lr){6-7} 
                      Loss & it/s       & mem         & it/s     & mem     & it/s      & mem     \\ \midrule
No XS                   & 12.1     & 4.4     & 38.5      & 2.5     & 11.6       & 4.4         \\ 
RRR                    & 2.6      & 9.7     & 12.5      & 3.2     & 3.0        & 9.2         \\ 
RBR                    & 1.7      & 8.6     & 7.3       & 4.0     & \multicolumn{2}{c}{OOM}  \\ 
TAP                    & \textbf{11.8}     & \textbf{4.4}     & \textbf{36.7}      & \textbf{2.5}     & \textbf{12.0}       & \textbf{4.5}         \\ \bottomrule 
\end{tabular}
\caption{Iterations per second (it/s) and memory usage in GB (mem) for batches of 16 $224\!\times\!224$-pixel images. DenseNet-121 with RBR runs out of memory (OOM). TAP adds less training overhead. Best result amongst XS methods in bold.}
\label{tab:train_speed}
\end{table}

\subsection{A Teacher-Student Framework}
Recognising spurious signals necessitates external input, as by definition they are informative in the training distribution. 
In previous works, the external input takes shape in the form of sample-wise ground-truth annotation masks $\{M_i\}_{i=1}^N$ \citep{rossRightRightReasons2017,riegerInterpretationsAreUseful2020,schramowskiMakingDeepNeural2020,shaoRightBetterReasons2021}, which can be costly when obtained from humans. To address this issue, we propose framing the objective of teaching a model to ignore spurious signals as a teacher-student problem. We consider the teacher as a model parameterised by $\theta_T$ which transfers its knowledge to a student model parameterised by $\theta_S$ to be trained on contaminated data, so that:
\begin{equation}
    \theta_S = \arg \min_\theta \sum_{i=1}^{N} \mathcal{L}(X_i,y_i,\theta_T;\theta)
\end{equation}

We instantiate this framework first with the standard setting in which the teacher model is a human who can provide ground-truth annotations and then with a pre-trained model which may provide noisier annotations. We argue that obtaining a small clean dataset $\{(X^*_i, y_i)\}_{i=1}^{N^*}$, with $N^*\!\ll\!N$, for pre-training a teacher may be easier than annotating all of the contaminated data. We propose using saliency maps to identify areas in the contaminated data the teacher model believes to be unimportant, which, if the teacher fulfils Equation~\ref{eq:unbiased}, should include the spurious signals. To find these areas, we make use of an explanation function $\mathcal{E}$ which accepts an NN and an input 
to produce a saliency map over the input: 
\begin{equation}
\begin{aligned}
    \mathcal{E}(X, y; \theta_T) = E 
\end{aligned}                        
\end{equation}
where $E \in \mathbb{R}^{H,W}$. In our experiments, we use input gradients as the explanation method, but the choice is not restricted to a particular saliency method. 
For notational convenience, we assume saliency maps $E$ contain absolute values normalised to $[0,1]$.
We propose using an element-wise threshold function $\Psi$ to construct the \emph{teacher annotation} $\Tilde{M}$:
\begin{equation}
\begin{aligned}
    \Tilde{M}_{ij}=(\Psi \circ E)_{ij} = \begin{cases} 1 & \text{if $E_{ij} < \tau$} \\
                            0 & \text{otherwise}
    \end{cases}
\end{aligned}                        
\end{equation}
for $i = 1,\ldots, H, j = 1,\ldots, W$. 
Intuitively, low saliency areas will be targeted by the teacher annotation. We can now re-state the optimal student parameters as:
\begin{equation}
    \theta_S = \arg \min_\theta \sum_{i=0}^{N}  \mathcal{L}(X_i,y_i,\Tilde{M_i};\theta)
\end{equation}

\section{Experimental Design}
\label{sec:experiments}
We first compare TAP against RRR and RBR with ground-truth annotations on a benchmark and two clinical datasets using four different CNN architectures. We then study 
performances under teacher annotations.

\subsection{Datasets}
We conduct experiments on MNIST, a standard benchmark in XS, and two clinical datasets: chest radiographs for detecting pneumonia (PNEU) and knee radiographs for grading osteoarthritis (KNEE). MNIST contains 60,000 $28\!\times\!28$ pixel images of handwritten digits. PNEU contains 5,232 paediatric chest radiographs, with approximately two-thirds presenting pneumonia. KNEE contains 8,260 radiographs of knee joints with 5 different osteoarthritis grades on an ordinal scale. The images in 
PNEU and KNEE vary in resolution and dimensions. We centre-crop and resize to $224\!\times\!224$ pixels. We use the pre-defined training and test splits for all datasets and reserve $10\%$ of the training split for validation.

To contaminate the datasets, we follow the common approach in XS and add spurious signals to the data which correlate with the label. 
Figure~\ref{fig:data} shows samples of the contaminations.
For MNIST, we follow \citet{rossRightRightReasons2017} and define $SC(X^*,y)$ (Definition~\ref{def:SC}) as adding $4\!\times\!4$ pixel patches to a random corner in $X^*$ with a pixel value set to $255 - label 
 \times 25$.
For $SC(X^*,permute(y))$ (the \emph{permuted} data), we choose random permutations of the labels, which essentially randomises the patch assignment.
For the medical datasets we add text strings to simulate medical tags. For PNEU, $SC(X^*,y)$ adds ``ABC'' for normal and ``XYZ'' for pneumonic cases. Given the binary prediction task, we choose to swap instead of randomising labels to construct $SC(X^*,permute(y))$. For KNEE, $SC(X^*,y)$ adds ``ABC'' (grade 0), ``DEF'' (grade 1), ``GHI'' (grade 2), ``JKL'' (grade 3) and ``MNO'' (grade 5), while $SC(X^*,permute(y))$ randomises the string assignment. We consider experiments with an additional contamination of the medical datasets in Appendix~\ref{appx:stripe}.

\begin{figure}[t]
\centering
\includegraphics[width=0.75\columnwidth]{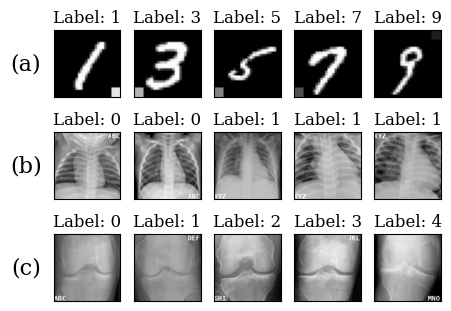} 
\caption{Overview of contaminated data. Brightness of corners correlate with label for (a) MNIST. Strings in corners correlate with label for (b) PNEU and (c) KNEE.}
\label{fig:data}
\end{figure}

For the teacher-student experiments, we hold out 10\% or 1\% of the training split as a clean dataset to pre-train a teacher before adding artifacts to induce spurious correlations in the remainder, as described above. 

\subsection{Evaluation Metrics}
If a model exploits spurious signals, it should perform worse when the correlations are removed or changed, so that Equation~\ref{eq:unbiased} does not hold. To evaluate the models, we (i) assess their performance on the contaminated data $SC(X^*,y)_{test}$. We measure accuracy for MNIST, F-score for PNEU given class imbalance, and mean absolute error (MAE) for KNEE given ordinal scores. We then (ii) assess the sensitivity of that initial performance to spurious signals by measuring the change in the respective metrics ($\Delta$-metric) when evaluated on $SC(X^*,permute(y))_{test}$. We report averaged results from 5 different seeds.

\subsection{Models}
We use a two-layer CNN (for architectural details, see Appendix~\ref{appx:architecture}) for MNIST and three commonly used CNN architectures for the medical datasets: VGG-16, ResNet-18 and DenseNet-121. 
For the teacher-student experiments, we use ResNet-18 as teacher for the medical datasets and the two-layer CNN for MNIST.

We apply TAP to both layers of the two-layer CNN. For ResNet-18, we found the output of each of the four residual blocks to be a natural choice. For DenseNet-121, we analogously chose the output of each of the four dense blocks. For VGG-16, we choose the output after each of the MaxPool layers. In each case, the targeted layers are spaced roughly evenly across the depth of the architectures.

We choose cross-entropy for the task loss $\mathcal{L}_{Task}$. We use SGD as optimiser with weight decay of 0.9. We train for 50 epochs. We use random initialisation for the two-layer CNN and use a learning rate of $10^{-3}$. For VGG-16, ResNet-18 and DenseNet-121 we initialise with ImageNet-weights and use a learning rate of $10^{-5}$. We use a batch size of 256 for MNIST and 16 for the medical datasets (8 for DenseNet-121 with RBR, given memory constraints). We tune $\lambda$ (Equation~\ref{eq:lambda}) by training on $SC(X^*,y)_{train}$ 
and evaluating the validation loss on  $SC(X^*,permute(y))_{val}$. We increase $\lambda$ in log steps from $10^{-9}$ to 1 for TAP and RRR and 1 to $10^9$ for RBR (given much smaller losses)---
see Appendix~\ref{appx:hyperparameters}. 

Experiments were implemented 
with PyTorch 1.13 and run on a Linux Ubuntu 18.04 machine with an Nvidia RTX 3080 GPU with 10GB VRAM.

\section{Results}
We show that TAP-trained models perform well on contaminated data and do not rely on spurious signals on 
the three chosen datasets.
Below we report experiments on the common XS setting with ground-truth annotations. Throughout, we use as baselines ``Base'' models trained on clean data and ``No XS'' models trained on contaminated data without XS. We then consider the second setting with teacher annotations. 
In the Appendix, we expand on our experiments, e.g., (i) we analyse input gradients and activations quantitatively (App~\ref{appx:quant_comparison}) and visually (App~\ref{appx:visual_comparison}), finding that TAP reduces the presence of spurious signals in deeper layers and, indirectly, in input gradients; (ii) we report results on another contamination of the medical datasets (App~\ref{appx:stripe}), further supporting our findings below; (iii) we report tabular values with standard deviations and (iv) include additional results on clean datasets (Apps~\ref{appx:gt_table}, \ref{appx:teacher_table}); (v) we provide further evidence of TAP's efficacy on additional CNNs (App~\ref{appx:additional_cnns}); (vi) we report an extended comparison of XS methods with teacher annotations (App~\ref{appx:rrr_rbr_teacher}); and, lastly, (vii) we discuss and compare with knowledge distillation methods (App~\ref{appx:kd_comparison}). 

\subsection{Setting 1: Ground-Truth Annotations}
\paragraph{MNIST.}
TAP performs competitively with existing methods. The left panel of Figure~\ref{fig:mnist} compares the test performance on $SC(X^*,y)_{test}$ of different XS losses (RRR, RBR, TAP) against Base and No XS models. At first glance, the different losses result in very similar performances. The right panel shows the drop in accuracy when evaluating on $SC(X^*,permute(y))_{test}$. Models that rely on spurious signals should be sensitive to changes in the spurious correlations. As expected, the accuracy of the No XS model drops significantly, whereas the Base model and the models trained with XS losses are insensitive to the spurious signals. This demonstrates that TAP performs as well as RRR and RBR.

\begin{figure}[h]
\centering
\includegraphics[width=0.99\columnwidth]{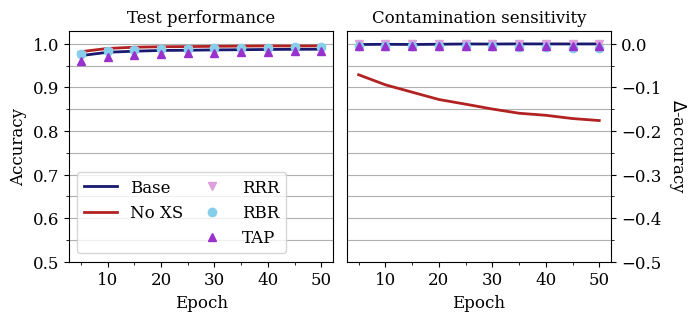} 
\caption{Test performance (accuracy) and contamination sensitivity ($\Delta$-accuracy) for MNIST with ground-truth annotations. Higher is better.}
\label{fig:mnist}
\end{figure}

\paragraph{PNEU.}
In Figure~\ref{fig:pneu}, the left panel initially suggest that the No XS model is best. The right panel, however, reveals that this is largely due to spurious signals, as the model's performance plummets when these change. The success of RRR and RBR vary with the model: RBR results in ostensibly strong performance for VGG-16, which turns out to rely on spurious signals, while RRR does so for DenseNet-121. 

TAP, on the other hand, closely matches the performance of the Base model and is insensitive to the change in spurious correlations for all three CNN architectures, highlighting its efficacy for deeper models, too.  

\begin{figure}[h]
\centering
\includegraphics[width=0.99\columnwidth]{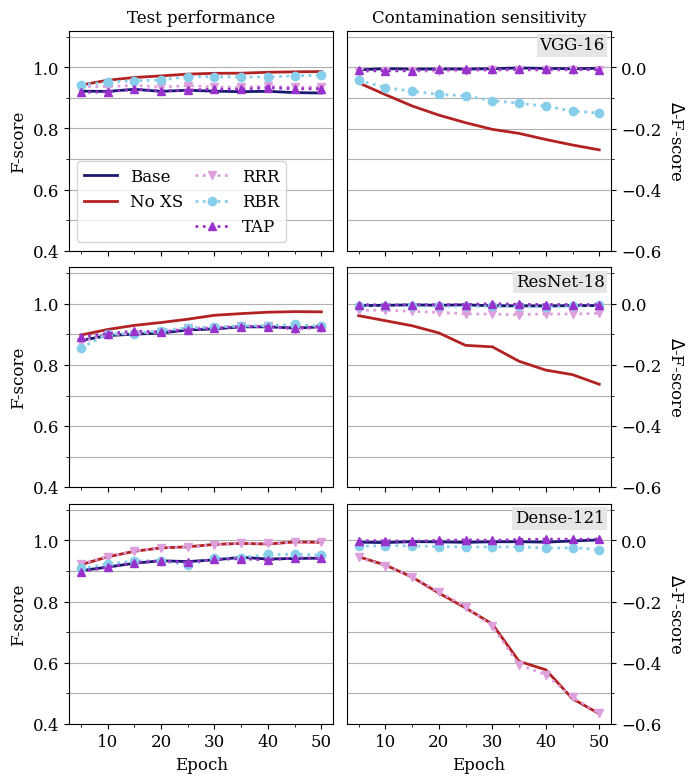} 
\caption{Test performance (F-score) and contamination sensitivity
($\Delta$-F-score) for PNEU with ground-truth annotations. Higher is better.}
\label{fig:pneu}
\end{figure}

\paragraph{KNEE.}

\begin{figure}[h]
\centering
\includegraphics[width=0.99\columnwidth]{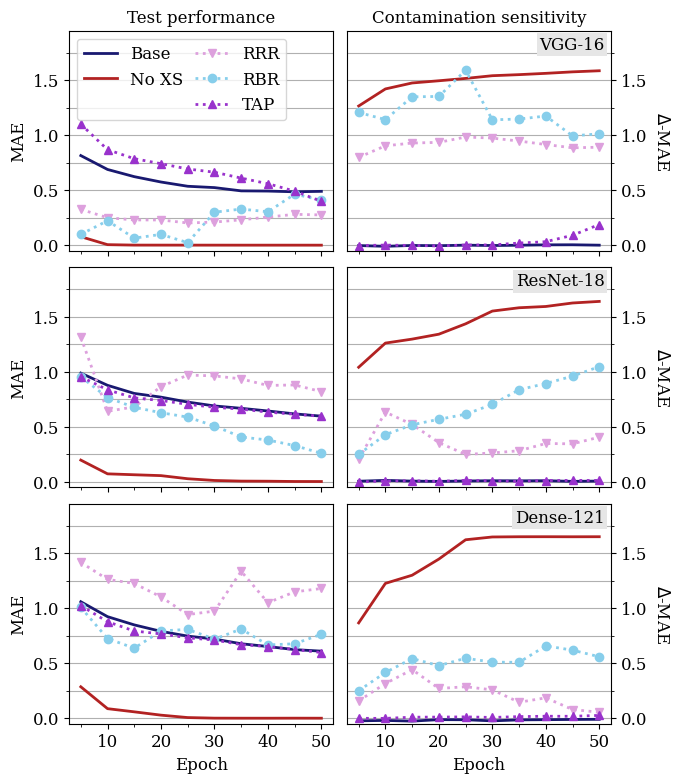} 
\caption{Test performance (MAE) and contamination sensitivity ($\Delta$-MAE) for KNEE with ground-truth annotations. Lower is better.}
\label{fig:knee}
\end{figure}

Figure~\ref{fig:knee} shows that while RRR and RBR seem to perform well (left panel) compared to TAP and even Base models, they exhibit significant contamination sensitivity (right panel) and thus only perform \emph{spuriously} well. Similarly, the No XS model at first appears to be best-performing but in fact relies on spurious signals. TAP, in contrast, shows little contamination sensitivity and matches the Base model. These results demonstrate the ability of TAP to protect deeper CNNs from learning spurious signals.

\subsection{Setting 2: Teacher Annotations}
Instead of ground-truth annotations, we now assume access to a small clean dataset (10\% or 1\% of the contaminated data in size) on which we pre-train a teacher to obtain annotations from. We focus on TAP and leave, for completeness, results for RRR and RBR to Appendix~\ref{appx:rrr_rbr_teacher}, which do not appear to work as well with teacher annotations compared to TAP. We also leave results on MNIST to Appendix~\ref{appx:teacher_mnist}.

Figure~\ref{fig:overlap} compares the teacher annotations to the ground-truth at different 
thresholds $\tau$.
The annotations are better than random (i.e., recall and precision of circa 1.9\%), and raising 
$\tau$
increases the recall and reduces precision. This implies (i) teacher annotations can identify spurious signals albeit imperfectly, and (ii) a higher $\tau$ results in capturing more spurious signals at the cost of including relevant areas. Interestingly, annotations from the 1\%-teacher are almost as good as those from the 10\%-teacher.

Figure~\ref{fig:ts} shows the final-epoch test performance and contamination sensitivity of using TAP with teacher and baseline random annotations. On PNEU, TAP with teacher annotations results in strong performances on the contaminated data with little sensitivity to spurious signals. Notably, there is very little difference between using the two sets of teacher annotations. The student models notably outperform both the 10\%-teacher (F-score: 0.869) and 1\%-teacher (F-score: 0.668). Random annotations, in contrast, result in performances which rely on spurious correlations for two of three 
models. This suggests that better-than-random but not necessarily perfect targeting of the spurious signals is needed.  

On KNEE, the success of using teacher annotations more heavily depends on (i) 
$\tau$, and (ii) the chosen 
model.
For DenseNet-121 and VGG-16, an increase in $\tau$ surprisingly results in an increase in contamination sensitivity (a similar trend can also be observed on PNEU). This, along with the poor performance of the random annotations, suggests that noisier annotations cannot be simply offset by penalising a greater area, underscoring again the importance of good targeting. Focusing therefore on $\tau=0.01$, we observe that only DenseNet-121 is successful in maintaining good performance which still suffers from some contamination sensitivity. VGG-16 achieves a seemingly low MAE which exploits spurious correlations. ResNet-18 did not train well with noisy annotations: while contamination sensitivity is low, this may be specious given its high MAE. 

Overall, these results suggest that teacher annotations can be effective substitutes depending on the 
task and model. Moreover, a small clean dataset may be sufficient to pre-train a teacher that can output useful annotations.

\begin{figure}[h]
\centering
\includegraphics[width=0.85\columnwidth]{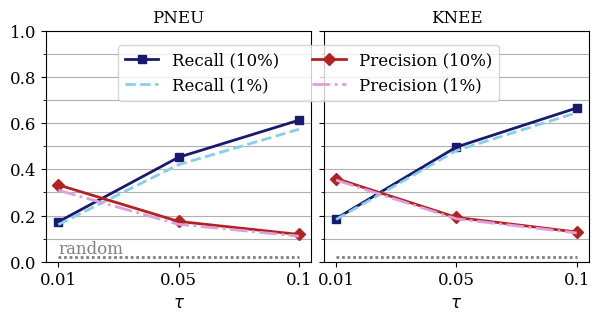} 
\caption{\emph{Annotation} precision and recall of teachers trained on  10\% or 1\% clean splits at different thresholds.}
\label{fig:overlap}
\end{figure}

\begin{figure}[h]
\centering
\includegraphics[width=0.99\columnwidth]{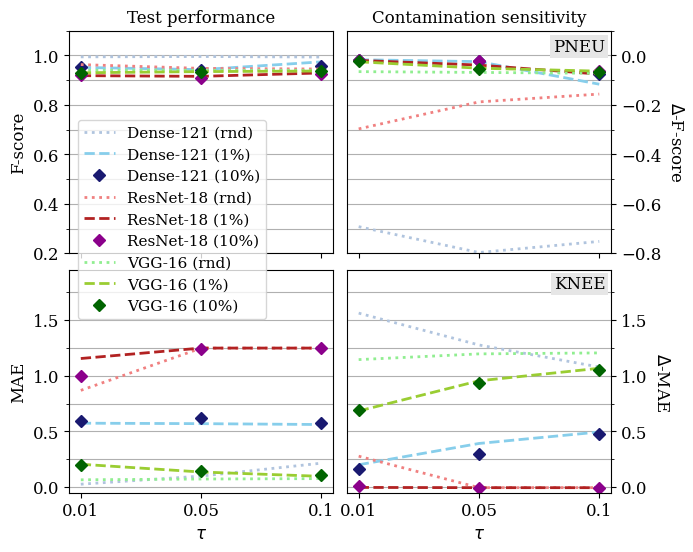} 
\caption{TAP test performance and contamination sensitivity 
with random (rnd) and teacher (1\% or 10\%) annotations vs ground-truth. 
Higher/Lower better for PNEU/KNEE.
}
\label{fig:ts}
\end{figure}

\section{Conclusion}
\label{sec:conclusion}
We have proposed TAP,  targeting activations of spurious signals with penalties as a way of mitigating their use in CNNs. TAP avoids expensive second-order derivatives, enabling faster training times and lower memory usage. We have shown TAP's relationship with the popular RRR and RBR losses and demonstrated its effectiveness on three datasets and four CNN architectures. We have also shown that TAP can maintain good performance with noisy teacher annotations which can help reduce otherwise costly ground-truth annotation requirements. We plan to investigate in future work if a variant of TAP can be generalised to other model architectures besides CNNs.

\section*{Acknowledgements}
We thank Hamed Ayoobi, Alice Giroul and anonymous reviewers for providing feedback. Dekai Zhang was funded by the UKRI CDT in AI for Healthcare \url{http://ai4health.io} (Grant No. EP/S023283/1).
Matthew Williams was funded through ICHT and the Imperial/NIHR BRC.
Francesca Toni was partially funded by the 
ERC under
the EU’s Horizon 2020 research and innovation programme (Grant Agreement No. 101020934) and 
by 
J.P. Morgan and  the
Royal Academy of Engineering under the Research Chairs and Senior Research Fellowships scheme.

\bibliography{refs}

\appendix
\onecolumn

\section*{Targeted Activation Penalties Help CNNs Ignore Spurious Signals \\ Appendix}
\section{Proof of Theorem~\ref{th:app:relationship}}
\label{appx:proof}
We first re-state Theorem 1 for clarity and provide the detailed proof below.

\setcounter{theorem}{0}
\begin{theorem}
\label{th:app:relationship}
    Suppose we have a CNN with convolutional layers $l =1,\ldots,L$, filters $w^l$ of kernel size $\kappa^l$ and ReLU activation $\sigma$. It is sufficient to optimise all activation map elements $A^{l}_{ab}$ with respect to $w^{l}$ to set the input gradient of pixel $(i,j)$ in layer $l$ (Equation~\ref{eq:input_grad_layer}) to zero, where $(a,b)$ are such that:
    \begin{equation*}
    \begin{aligned}
        i - \kappa^1 - \ldots - \kappa^l + l \leq a \leq i \\
        j - \kappa^1 - \ldots - \kappa^l + l \leq b \leq j
    \end{aligned}
    \label{app:eq:receptive}
    \end{equation*}
\end{theorem}

\begin{proof}
    We first derive the area (Equation~\ref{eq:receptive}, also re-stated directly above) within which the input gradient with respect to input $(i,j)$ can be non-zero. All products between matrices or tensors are convolutional products, so we suppress the $*$-operator for brevity. Recalling that $Z^l = w^l  \sigma(Z^{l-1}) + b^l$, with filters $w^l$ of kernel size $\kappa^l$ and bias $b^l$, and letting $Z^0 = X$, the input gradient of pixel $(i,j)$ in layer $l$ is then:
    \begin{equation}
    \label{eq:input_gradient}
    \begin{aligned}
        \frac{\partial }{\partial X_{ij}} \sigma(Z^l) 
            &= \sigma'(Z^l) \odot \frac{\partial }{\partial X_{ij}} \left[ w^l  \sigma(Z^{l-1}) + b^l \right] \\
            &= \sigma'(Z^l) \odot \frac{\partial }{\partial X_{ij}} w^l  \sigma(Z^{l-1}) \\
            &=  \sigma'(Z^l) \odot w^l  \left(\sigma'(Z^{l-1}) \odot \frac{\partial Z^{l-1}}{\partial X_{ij}} \right)   \\
            &=  \sigma'(Z^l) \odot w^l \sigma'(Z^{l-1})
                        \odot  w^l w^{l-1} \sigma'(Z^{l-2})
                        \odot \ldots
                        \odot w^l w^{l-1} w^{l-2} \ldots w^{1} \sigma'(X)
                        \odot w^l w^{l-1} w^{l-2} \ldots w^{1} \frac{\partial X}{\partial X_{ij}}
    \end{aligned}
    \end{equation}

    Focusing on the last term of the Hadamard product, we can observe that for any $(i,j)$, we must have for a given activation feature $(a,b)$:
        
    \begin{equation*}
    \begin{aligned}
        &\left(w^l w^{l-1} w^{l-2} \ldots w^{1} \frac{\partial X}{\partial X_{ij}} \right)_{ab} \\
            &= \sum_{m=1}^{\kappa^l} \sum_{n=1}^{\kappa^l} w^l_{mn} \left( w^{l-1} w^{l-2} \ldots w^{1} \frac{\partial X}{\partial X_{ij}} \right)_{a+m-1, b+n-1} \\
            &= \sum_{m=1}^{\kappa^l} \sum_{n=1}^{\kappa^l} w^l_{mn} 
                \sum_{m'=1}^{\kappa^{l-1}} \sum_{n'=1}^{\kappa^{l-1}} w^{l-1}_{m'n'}
                \left( w^{l-2} \ldots w^{1} \frac{\partial X}{\partial X_{ij}} \right)_{a+m-1+m'-1, b+n-1+n'-1} \\
            &= \sum_{m=1}^{\kappa^l} \sum_{n=1}^{\kappa^l} w^l_{mn} 
                \sum_{m'=1}^{\kappa^{l-1}} \sum_{n'=1}^{\kappa^{l-1}} w^{l-1}_{m'n'}
                \ldots
                \sum_{m^{l-1}=1}^{\kappa^{1}} \sum_{n^{l-1}=1}^{\kappa^{1}} w^{1}_{m^{l-1}n^{l-1}}
                \left( \frac{\partial X}{\partial X_{ij}} \right)_{a+m+m'+\ldots+m^{l-1}-l, b+n+n'+\ldots+n^{l-1}-l}
    \end{aligned}
    \end{equation*}
    where the last term can only be non-zero when $(a,b)$ are such that $(i,j)$ is captured within its receptive field:
    \begin{equation*}
    \begin{aligned}
        a \leq i \leq a + \kappa^1 + \ldots + \kappa^l - l \\
        b \leq j \leq b + \kappa^1 + \ldots + \kappa^l - l
    \end{aligned}
    \end{equation*}
    which we can re-write to give us Equation~\ref{eq:receptive}:
    \begin{equation*}
    \begin{aligned}
        i - \kappa^1 - \ldots - \kappa^l + l \leq a \leq i  \\
        j - \kappa^1 - \ldots - \kappa^l + l \leq b \leq j
    \end{aligned}
    \end{equation*}

    We next show that optimising all of the activations elements $(a,b)$ that fulfil Equation~\ref{eq:receptive} of layer $l$ with respect to its filters $w^l$ is sufficient to set the input gradient of pixel $(i,j)$ in that layer to zero. Note that Equation~\ref{eq:input_gradient} implies that it suffices to set $\sigma'(Z^{l})_{a,b}=0$. The first-order condition for filters $w^l_c$, where $c=1,\ldots,C^l$, sets the following to zero:

    \begin{equation*}
    \begin{aligned}
        \frac{\partial A^l_{ab}}{\partial w^l_c} 
        &= \frac{\partial }{\partial w^l_c} \sum_{k=1}^{C^l} \sigma(Z^l_{kab}) \\
        &= \frac{\partial }{\partial w^l_c} \sigma(Z^l_{cab}) \\
        &= \frac{\partial }{\partial w^l_c} \sigma\left(
            \sum_{m=1}^{\kappa^l} \sum_{n=1}^{\kappa^l} w^l_{cmn} \sigma(Z^{l-1}_{a+m-1,b+n-1}) + b^l_{jab} 
            \right) \\
        &= \sigma'(Z^l_{cab})
            \begin{bmatrix}
                \sigma(Z^{l-1}_{a,b}) & \dots & \sigma(Z^{l-1}_{a+\kappa^l-1,b}) \\
                \vdots & \ddots & \vdots \\
                \sigma(Z^{l-1}_{a,b+\kappa^l-1}) & \dots & \sigma(Z^{l-1}_{a+\kappa^l-1,b+\kappa^l-1})     
            \end{bmatrix}
    \end{aligned}
    \end{equation*}

    As the output of the preceding layer is constant with respect to $w^l_c$, in the optimum with non-zero activations $Z^{l-1}$ and ReLU activation function $\sigma$, we must have that $\sigma'(Z^l_{cab}) = 0$ for all $c=1,\ldots,C^l$ and all $(a,b)$ that fulfil Equation~$\ref{eq:receptive}$. Therefore $\partial \sigma(Z^l) / \partial X_{ij} = 0$.   

\end{proof}

\clearpage
\section{Architecture of Two-Layer CNN}
\label{appx:architecture}

\begin{figure*}[h!]
\centering
\includegraphics[width=0.8\textwidth]{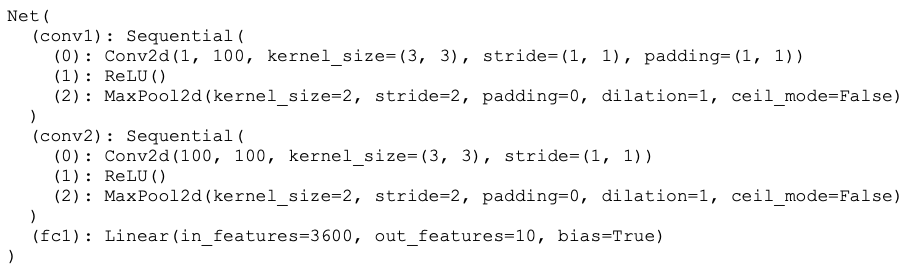} 
\caption{Overview of two-layer CNN architecture.}
\label{app:fig:two_layer}
\end{figure*}

\section{Hyperparameter Choices}
\label{appx:hyperparameters}

\begin{table*}[h]
\centering
\begin{tabular}{lcccccccccccc}
\toprule
       & \multicolumn{3}{c}{DenseNet-121}  & \multicolumn{3}{c}{ResNet-18}     & \multicolumn{3}{c}{VGG-16}   & \multicolumn{3}{c}{Two-layer CNN}         \\ \cmidrule(lr){2-4}\cmidrule(lr){5-7}\cmidrule(lr){8-10}\cmidrule(lr){11-13} 
     Data & TAP       & RRR       & RBR      & TAP      & RRR      & RBR     & TAP     & RRR     & RBR      & TAP     & RRR     & RBR \\ \midrule
PNEU & $10^{-4}$     & $10^{-7}$     & $10^{6}$    & $10^{-3}$    & $10^{-3}$    & $10^{7}$   & $10^{-7}$   & $10^{-3}$   & $10^{5}$ & - & - & - \\
KNEE & $10^{-5}$     & $10^{-2}$     & $10^{5}$    & $10^{-5}$    & $10^{-2}$    & $10^{5}$   & $10^{-6}$   & $10^{-2}$   & $10^{3}$ & - & - & - \\ 
MNIST & - & - & - & - & - & - & - & - & - & $10^{-3}$ & $10^{-3}$ & $10^{4}$ \\ \bottomrule
\end{tabular}
\caption{Choices for $\lambda$ for all XS losses with ground-truth annotations.}
\end{table*}

\begin{table*}[h]
\centering
\begin{tabular}{lcccccccccccc}
\toprule
       & \multicolumn{3}{c}{DenseNet-121}  & \multicolumn{3}{c}{ResNet-18}     & \multicolumn{3}{c}{VGG-16}   & \multicolumn{3}{c}{Two-layer CNN}         \\ \cmidrule(lr){2-4}\cmidrule(lr){5-7}\cmidrule(lr){8-10}\cmidrule(lr){11-13} 
     Data & TAP       & RRR       & RBR      & TAP      & RRR      & RBR     & TAP     & RRR     & RBR      & TAP     & RRR     & RBR \\ \midrule
PNEU  & $10^{-6}$ & $10^{-8}$ & $10^{5}$ & $10^{-6}$ & $10^{-4}$ & $10^{2}$ & $10^{-7}$ & $10^{-3}$ & $10^{0}$ & -     & -     & -     \\
KNEE  & $10^{-5}$ & $10^{-2}$ & $10^{3}$ & $10^{-5}$ & $10^{-1}$ & $10^{4}$ & $10^{-6}$ & $10^{-2}$ & $10^{9}$ & -     & -     & -     \\
MNIST & -     & -     & -     & -     & -     & -     & -     & -     & -     & $10^{-4}$ & $10^{-2}$ & $10^{4}$ \\ \bottomrule
\end{tabular}
\caption{Choices for $\lambda$ for all XS losses with teacher annotations. TAP results reported in main paper, RRR and RBR results reported in Appendix~\ref{appx:rrr_rbr_teacher}.}
\end{table*}

\clearpage
\section{Quantitative Comparison of Input Gradients and Activations}
\label{appx:quant_comparison}

In this section, we evaluate how successful the XS methods (RRR, RBR, TAP) are at mitigating the influence of spurious signals by finding the overlap between the location of spurious signals and the location of the top-quartile of input gradients and activations by magnitude. We consider the activations of the four penalised convolutional layers (the last four of five for VGG-16). As the activations are often smaller than the input, we upsample to match dimensions. We report recall and precision:
\begin{equation}
\begin{aligned}
    Recall = \frac{|(\text{spurious} \cap \text{top-quartile})|}{|\text{spurious}|} & & &    Precision = \frac{|(\text{spurious} \cap \text{top-quartile})|}{|\text{top-quartile}|}    
\end{aligned}
\end{equation}

Intuitively, we would expect that a model that is not susceptible to spurious signals to have very little overlap between its top-quartile input gradients / activations and spurious signals. Tables~\ref{app:table:overlap_pneu} and \ref{app:table:overlap_knee} show that TAP successfully suppresses the presence of spurious signals in the activations, as both recall and precision are either low throughout or decrease towards zero in the deeper layers (such as for DenseNet-121 and VGG-16). Notably, TAP suppresses the presence of spurious signals in the input gradients, even though it does not directly target them.

\begin{table*}[h]
\centering
\resizebox{\textwidth}{!}{
\begin{tabular}{llcccccccccc}
\toprule
                           & \multicolumn{1}{l}{} & \multicolumn{5}{c}{Recall}                                                                                                                                                      & \multicolumn{5}{c}{Precision}                                                                                                                                                   \\ \cmidrule(lr){3-7} \cmidrule(lr){8-12}
Architecture               & Case                 & \multicolumn{1}{c}{IG} & \multicolumn{1}{c}{Act 1} & \multicolumn{1}{c}{Act 2} & \multicolumn{1}{c}{Act 3} & \multicolumn{1}{c}{Act 4} & \multicolumn{1}{c}{IG} & \multicolumn{1}{c}{Act 1} & \multicolumn{1}{c}{Act 2} & \multicolumn{1}{c}{Act 3} & \multicolumn{1}{c}{Act 4} \\ \midrule
\multirow{5}{*}{Dense-121} & Base                 & 0.01 (0.00)                         & 0.80 (0.03)                      & 0.60 (0.05)                      & 0.61 (0.03)                      & 0.20 (0.02)                      & 0.00 (0.00)                         & 0.06 (0.00)                      & 0.05 (0.00)                      & 0.05 (0.00)                      & 0.02 (0.00)                      \\
                           & No XS                & 0.08 (0.03)                         & 0.77 (0.08)                      & 0.55 (0.03)                      & 0.41 (0.09)                      & 0.88 (0.17)                      & 0.01 (0.00)                         & 0.06 (0.01)                      & 0.04 (0.00)                      & 0.03 (0.01)                      & 0.07 (0.01)                      \\
                           & RRR                  & 0.08 (0.03)                         & 0.76 (0.08)                      & 0.55 (0.04)                      & 0.41 (0.10)                      & 0.89 (0.13)                      & 0.01 (0.00)                         & 0.06 (0.01)                      & 0.04 (0.00)                      & 0.03 (0.01)                      & 0.07 (0.01)                      \\
                           & RBR                  & 0.01 (0.01)                         & 0.56 (0.09)                      & 0.42 (0.21)                      & 0.48 (0.28)                      & 0.33 (0.10)                      & \textbf{0.00 (0.00)}                & 0.04 (0.01)                      & 0.03 (0.02)                      & 0.04 (0.02)                      & 0.03 (0.01)                      \\
                           & TAP                  & \textbf{0.00 (0.00)}                & \textbf{0.44 (0.06)}             & \textbf{0.22 (0.01)}             & \textbf{0.10 (0.03)}             & \textbf{0.00 (0.00)}             & \textbf{0.00 (0.00)}                & \textbf{0.03 (0.00)}             & \textbf{0.02 (0.00)}             & \textbf{0.01 (0.00)}             & \textbf{0.00 (0.00)}             \\ \cmidrule(lr){1-12}
\multirow{5}{*}{ResNet-18} & Base                 & 0.00 (0.00)                         & 0.97 (0.01)                      & 0.82 (0.02)                      & 0.84 (0.01)                      & 0.06 (0.02)                      & 0.00 (0.00)                         & 0.07 (0.00)                      & 0.06 (0.00)                      & 0.06 (0.00)                      & 0.00 (0.00)                      \\
                           & No XS                & 0.08 (0.03)                         & 0.99 (0.00)                      & 0.88 (0.02)                      & 0.95 (0.01)                      & 0.74 (0.14)                      & 0.01 (0.00)                         & 0.08 (0.00)                      & 0.07 (0.00)                      & 0.07 (0.00)                      & 0.06 (0.01)                      \\
                           & RRR                  & 0.02 (0.00)                         & 0.98 (0.01)                      & 0.89 (0.01)                      & 0.94 (0.01)                      & 0.41 (0.14)                      & \textbf{0.00 (0.00)}                & 0.08 (0.00)                      & 0.07 (0.00)                      & 0.07 (0.00)                      & 0.03 (0.01)                      \\
                           & RBR                  & 0.03 (0.05)                         & 0.92 (0.04)                      & 0.86 (0.03)                      & 0.76 (0.11)                      & 0.19 (0.26)                      & \textbf{0.00 (0.00)}                & 0.07 (0.00)                      & 0.07 (0.00)                      & 0.06 (0.01)                      & 0.01 (0.02)                      \\
                           & TAP                  & \textbf{0.00 (0.00)}                & \textbf{0.00 (0.00)}             & \textbf{0.00 (0.00)}             & \textbf{0.00 (0.00)}             & \textbf{0.00 (0.00)}             & \textbf{0.00 (0.00)}                & \textbf{0.00 (0.00)}             & \textbf{0.00 (0.00)}             & \textbf{0.00 (0.00)}             & \textbf{0.00 (0.00)}             \\ \cmidrule(lr){1-12}
\multirow{5}{*}{VGG-16}    & Base                 & 0.01 (0.00)                         & 0.96 (0.00)                      & 0.97 (0.00)                      & 0.97 (0.00)                      & 0.06 (0.01)                      & 0.00 (0.00)                         & 0.07 (0.00)                      & 0.08 (0.00)                      & 0.07 (0.00)                      & 0.00 (0.00)                      \\
                           & No XS                & 0.26 (0.02)                         & 0.96 (0.00)                      & 0.99 (0.00)                      & 0.99 (0.00)                      & 0.84 (0.03)                      & 0.02 (0.00)                         & 0.07 (0.00)                      & 0.08 (0.00)                      & 0.08 (0.00)                      & 0.07 (0.00)                      \\
                           & RRR                  & \textbf{0.00 (0.00)}                & 0.96 (0.00)                      & 0.97 (0.00)                      & 0.97 (0.00)                      & 0.05 (0.01)                      & \textbf{0.00 (0.00)}                & 0.07 (0.00)                      & 0.08 (0.00)                      & 0.08 (0.00)                      & \textbf{0.00 (0.00)}             \\
                           & RBR                  & 0.08 (0.01)                         & 0.96 (0.00)                      & 0.99 (0.00)                      & 1.00 (0.00)                      & 0.86 (0.02)                      & 0.01 (0.00)                         & 0.07 (0.00)                      & 0.08 (0.00)                      & 0.08 (0.00)                      & 0.07 (0.00)                      \\
                           & TAP                  & 0.01 (0.00)                         & \textbf{0.67 (0.00)}             & \textbf{0.11 (0.01)}             & \textbf{0.21 (0.02)}             & \textbf{0.01 (0.00)}             & \textbf{0.00 (0.00)}                & \textbf{0.05 (0.00)}             & \textbf{0.01 (0.00)}             & \textbf{0.02 (0.00)}             & \textbf{0.00 (0.00)}  \\ \bottomrule
\end{tabular}} 
\caption{Recall and precision of top-quartile input gradients (IG) and activations at four layers (Act 1--4) in capturing spurious signals for PNEU. Lower values equate to less overlap with spurious signals. Base was trained without XS on clean data, No XS was trained without XS on contaminated data. Lowest overlap with spurious signals amongst XS methods (RRR, RBR, TAP) in bold. TAP minimises overlap between top-quartile activations and spurious signals and indirectly achieves the same for input gradients.}
\label{app:table:overlap_pneu}
\end{table*}

\begin{table*}[h]
\centering
\resizebox{\textwidth}{!}{
\begin{tabular}{llcccccccccc}
\toprule
                           & \multicolumn{1}{l}{} & \multicolumn{5}{c}{Recall}                                                                                                                                                      & \multicolumn{5}{c}{Precision}                                                                                                                                                   \\ \cmidrule(lr){3-7} \cmidrule(lr){8-12}
Architecture               & Case                 & \multicolumn{1}{c}{IG} & \multicolumn{1}{c}{Act 1} & \multicolumn{1}{c}{Act 2} & \multicolumn{1}{c}{Act 3} & \multicolumn{1}{c}{Act 4} & \multicolumn{1}{c}{IG} & \multicolumn{1}{c}{Act 1} & \multicolumn{1}{c}{Act 2} & \multicolumn{1}{c}{Act 3} & \multicolumn{1}{c}{Act 4} \\ \midrule
\multirow{5}{*}{Dense-121} & Base                 & 0.00 (0.00)                         & 0.86 (0.01)                      & 0.80 (0.04)                      & 0.81 (0.02)                      & 0.18 (0.04)                      & 0.00 (0.00)                         & 0.07 (0.00)                      & 0.06 (0.00)                      & 0.06 (0.00)                      & 0.01 (0.00)                      \\
                           & No XS                & 0.33 (0.01)                         & 0.83 (0.02)                      & 0.54 (0.03)                      & 0.58 (0.06)                      & 1.00 (0.00)                      & 0.03 (0.00)                         & 0.06 (0.00)                      & 0.04 (0.00)                      & 0.05 (0.00)                      & 0.08 (0.00)                      \\
                           & RRR                  & 0.29 (0.08)                         & 0.98 (0.03)                      & 0.95 (0.10)                      & 0.88 (0.08)                      & 1.00 (0.00)                      & 0.02 (0.01)                         & 0.08 (0.00)                      & 0.07 (0.01)                      & 0.07 (0.01)                      & 0.08 (0.00)                      \\
                           & RBR                  & 0.66 (0.15)                         & 0.80 (0.28)                      & 0.86 (0.30)                      & \textbf{0.81 (0.18)}             & 0.99 (0.02)                      & 0.05 (0.01)                         & \textbf{0.06 (0.02)}             & 0.07 (0.02)                      & \textbf{0.06 (0.01)}             & 0.08 (0.00)                      \\
                           & TAP                  & \textbf{0.00 (0.00)}                & \textbf{0.74 (0.03)}             & \textbf{0.79 (0.03)}             & 0.99 (0.00)                      & \textbf{0.00 (0.00)}             & \textbf{0.00 (0.00)}                & \textbf{0.06 (0.00)}             & \textbf{0.06 (0.00)}             & 0.08 (0.00)                      & \textbf{0.00 (0.00)}             \\ \cmidrule(lr){1-12}
\multirow{5}{*}{ResNet-18} & Base                 & 0.00 (0.00)                         & 0.99 (0.00)                      & 0.93 (0.01)                      & 0.94 (0.00)                      & 0.07 (0.03)                      & 0.00 (0.00)                         & 0.08 (0.00)                      & 0.07 (0.00)                      & 0.07 (0.00)                      & 0.01 (0.00)                      \\
                           & No XS                & 0.35 (0.03)                         & 1.00 (0.00)                      & 0.99 (0.00)                      & 1.00 (0.00)                      & 1.00 (0.00)                      & 0.03 (0.00)                         & 0.08 (0.00)                      & 0.08 (0.00)                      & 0.08 (0.00)                      & 0.08 (0.00)                      \\
                           & RRR                  & 0.22 (0.04)                         & 1.00 (0.00)                      & 1.00 (0.00)                      & 1.00 (0.00)                      & 0.75 (0.15)                      & 0.02 (0.00)                         & 0.08 (0.00)                      & 0.08 (0.00)                      & 0.08 (0.00)                      & 0.06 (0.01)                      \\
                           & RBR                  & 0.52 (0.11)                         & 0.97 (0.01)                      & 0.99 (0.00)                      & 1.00 (0.00)                      & 0.99 (0.01)                      & 0.04 (0.01)                         & 0.08 (0.00)                      & 0.08 (0.00)                      & 0.08 (0.00)                      & 0.08 (0.00)                      \\
                           & TAP                  & \textbf{0.00 (0.00)}                & \textbf{0.03 (0.00)}             & \textbf{0.09 (0.00)}             & \textbf{0.21 (0.01)}             & \textbf{0.00 (0.00)}             & \textbf{0.00 (0.00)}                & \textbf{0.00 (0.00)}             & \textbf{0.01 (0.00)}             & \textbf{0.02 (0.00)}             & \textbf{0.00 (0.00)}             \\ \cmidrule(lr){1-12}
\multirow{5}{*}{VGG-16}    & Base                 & 0.01 (0.00)                         & 0.99 (0.00)                      & 1.00 (0.00)                      & 0.99 (0.00)                      & 0.02 (0.00)                      & 0.00 (0.00)                         & 0.08 (0.00)                      & 0.08 (0.00)                      & 0.08 (0.00)                      & 0.00 (0.00)                      \\
                           & No XS                & 0.66 (0.01)                         & 0.99 (0.00)                      & 1.00 (0.00)                      & 1.00 (0.00)                      & 1.00 (0.00)                      & 0.05 (0.00)                         & 0.08 (0.00)                      & 0.08 (0.00)                      & 0.08 (0.00)                      & 0.08 (0.00)                      \\
                           & RRR                  & \textbf{0.02 (0.00)}                & 0.99 (0.00)                      & 1.00 (0.00)                      & 1.00 (0.00)                      & 1.00 (0.00)                      & \textbf{0.00 (0.00)}                & 0.08 (0.00)                      & 0.08 (0.00)                      & 0.08 (0.00)                      & 0.08 (0.00)                      \\
                           & RBR                  & 0.64 (0.23)                         & 0.99 (0.00)                      & 1.00 (0.00)                      & 1.00 (0.00)                      & 0.97 (0.07)                      & 0.05 (0.02)                         & 0.08 (0.00)                      & 0.08 (0.00)                      & 0.08 (0.00)                      & 0.08 (0.01)                      \\
                           & TAP                  & 0.07 (0.10)                         & \textbf{0.60 (0.05)}             & \textbf{0.23 (0.08)}             & \textbf{0.36 (0.20)}             & \textbf{0.06 (0.10)}             & 0.01 (0.01)                         & \textbf{0.05 (0.00)}             & \textbf{0.02 (0.01)}             & \textbf{0.03 (0.02)}             & \textbf{0.00 (0.01)}   \\ \bottomrule
\end{tabular}} 
\caption{Recall and precision of top-quartile input gradients (IG) and activations at four layers (Act 1--4) in capturing spurious signals for KNEE. Lower values equate to less overlap with spurious signals. Base was trained without XS on clean data, No XS was trained without XS on contaminated data. Lowest overlap with spurious signals amongst XS methods (RRR, RBR, TAP) in bold. TAP minimises overlap between top-quartile activations and spurious signals and indirectly achieves the same for input gradients.}
\label{app:table:overlap_knee}
\end{table*}

\clearpage
\section{Additional Visual Comparison of Input Gradients and Activations}
\label{appx:visual_comparison}

\begin{figure}[h!]
\centering
\includegraphics[height=0.86\textheight]{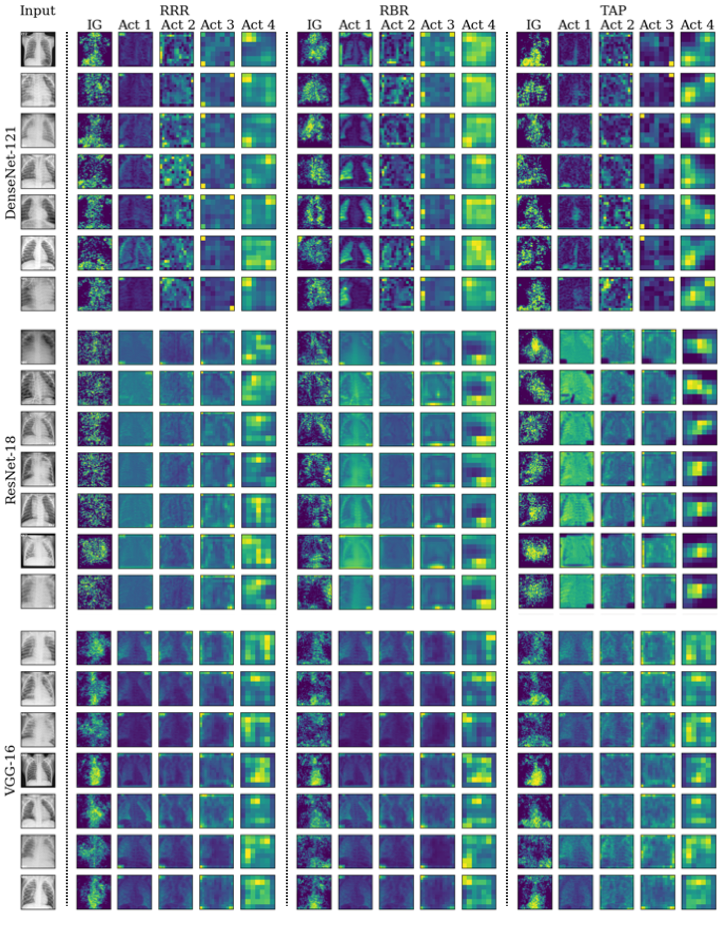} 
\caption{Randomly chosen PNEU images. Input gradients and activations at penalised layers of three different architectures. With RRR and RBR, input gradients ostensibly suggest that spurious signals are suppressed, but intermediate activations show spurious signals (see corners) can re-emerge in deeper layers. In contrast, TAP successfully mitigates their re-emergence which notably suppresses their input gradients as well.}
\label{app:fig:explanations-pneu}
\end{figure}

\begin{figure*}[h!]
\centering
\includegraphics[height=0.86\textheight]{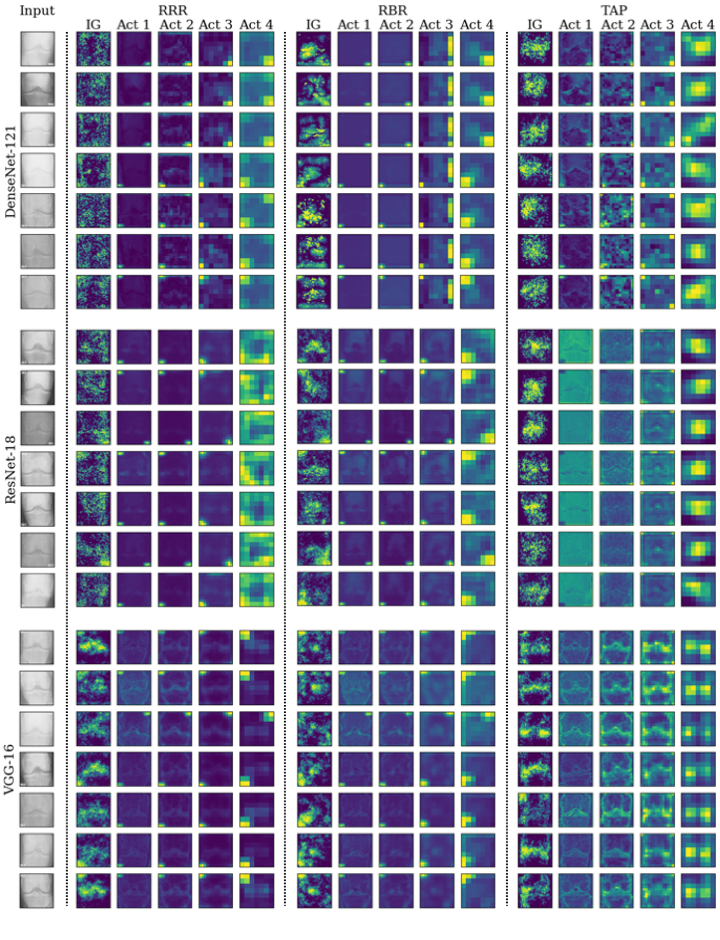} 
\caption{Randomly chosen KNEE images. Input gradients and activations at penalised layers of three different architectures. With RRR and RBR, input gradients ostensibly suggest that spurious signals are suppressed, but intermediate activations show spurious signals (see corners) can re-emerge in deeper layers. In contrast, TAP successfully mitigates their re-emergence which notably suppresses their input gradients as well.}
\label{app:fig:explanations-knee}
\end{figure*}

\clearpage
\section{Additional Experiments with Stripe Artifact}
\label{appx:stripe}

In this section we describe experiments on an additional contamination of the two medical datasets. Instead of adding text artifacts, we add a vertical stripe to the sides of the images (Figure~\ref{app:stripe_data}). The stripes are 6 pixels wide and randomly placed up to 15 pixels away from the edge. For PNEU, we set the brightness of the stripe to $255 - label * 127$ for $SC(X^*, y)$. We swap the assignment between the two labels for the permuted version $SC(X^*, permute(y))$. For KNEE, we set the brightness of the artifact to $255 - label * 51$ for $SC(X^*, y)$. We randomise the assigment between labels for the permuted version $SC(X^*, permute(y))$. We tune $\lambda$ as before. Table~\ref{app:table:stripe_paras} states the chosen values.

\begin{figure*}[h!]
\centering
\includegraphics[width=0.6\textwidth]{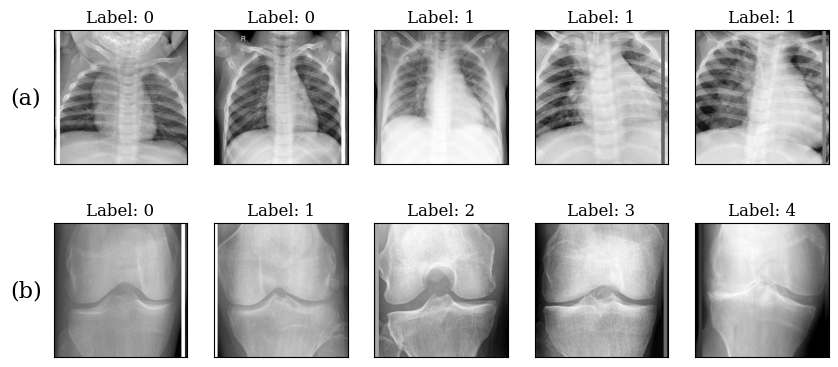} 
\caption{Overview of (a) PNEU and (b) KNEE with stripe artifacts.}
\label{app:stripe_data}
\end{figure*}

Figure~\ref{app:fig:pneu_stripe} shows that TAP is successful in protecting models from learning spurious signals on PNEU. Save for an uptick in contamination sensitivity with DenseNet-121, TAP is also successful in protecting models on KNEE. We state detailed tabular values for contaminated, permuted and clean data in Table~\ref{app:table:pneu_stripe} for PNEU and Table~\ref{app:table:knee_stripe} for KNEE.

\begin{figure*}[ht]
\centering
\begin{subfigure}{0.5\columnwidth}
\centering
\includegraphics[height=\textwidth]{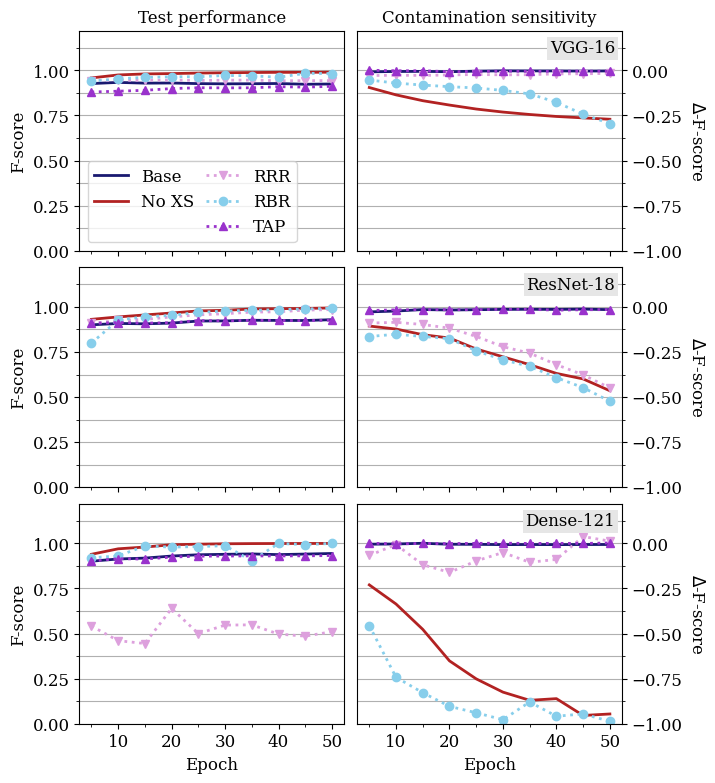}     
\caption{PNEU}
\end{subfigure}%
\begin{subfigure}{0.5\columnwidth}
\centering
\includegraphics[height=\textwidth]{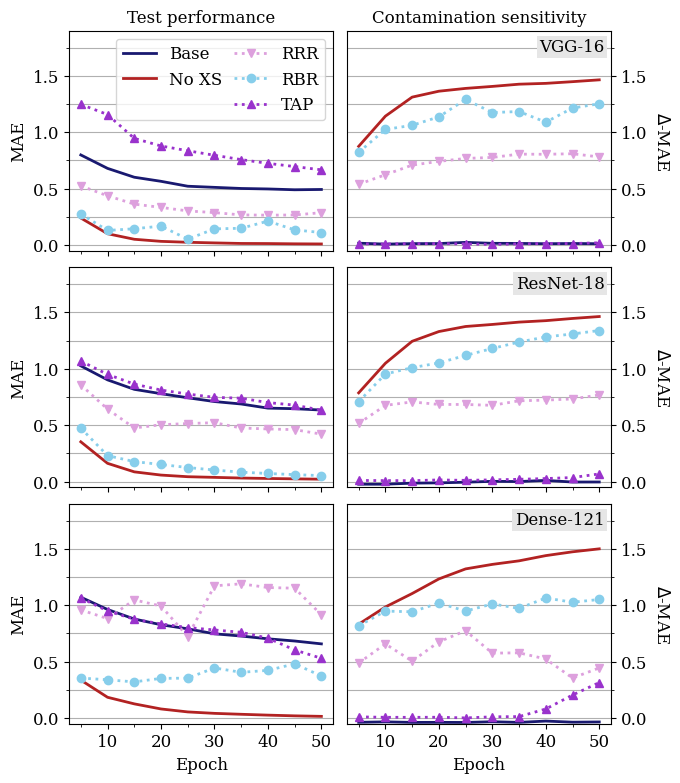}
\caption{KNEE}
\end{subfigure}
\caption{Results on PNEU and KNEE with stripe artifacts and ground-truth annotations. Comparisons of XS losses (RRR, RBR, TAP) against a Base model trained on clean data and a No XS model trained without XS. Higher is better for PNEU, lower is better for KNEE.}
\label{app:fig:pneu_stripe}
\end{figure*}

\clearpage
\begin{table*}[ht]
\centering
\begin{tabular}{lccccccccc}
\toprule
       & \multicolumn{3}{c}{DenseNet-121}  & \multicolumn{3}{c}{ResNet-18}     & \multicolumn{3}{c}{VGG-16}   \\ \cmidrule(lr){2-4}\cmidrule(lr){5-7}\cmidrule(lr){8-10}
     Data & TAP       & RRR       & RBR      & TAP      & RRR      & RBR     & TAP     & RRR     & RBR      \\ \midrule
PNEU & $10^{-4}$ & $10^{0}$ & $10^{3}$ & $10^{-5}$ & $10^{-4}$ & $10^{5}$ & $10^{-5}$ & $10^{-3}$ & $10^{6}$ \\
KNEE & $10^{-2}$ & $10^{-4}$ & $10^{4}$ & $10^{-3}$ & $10^{-2}$ & $10^{3}$ & $10^{-5}$ & $10^{-1}$ & $10^{4}$ \\ \bottomrule
\end{tabular}
\caption{Choices for $\lambda$ for XS losses with ground-truth annotations. PNEU and KNEE with stripe artifacts.}
\label{app:table:stripe_paras}
\end{table*}

We present tabular results in Tables~\ref{app:fig:pneu_stripe} and \ref{app:table:knee_stripe} on three versions of the test split: (i) the contaminated version $SC(X^*,y)$, (ii) the version with the permuted spurious signals $SC(X^*,permute(y))$ and (iii) a clean version $X^*$. 

\begin{table*}[ht]
\centering
\resizebox{\textwidth}{!}{
\begin{tabular}{@{}lccccccccc@{}}
\toprule
       & \multicolumn{3}{c}{DenseNet-121}  & \multicolumn{3}{c}{ResNet-18}     & \multicolumn{3}{c}{VGG-16}        \\ \cmidrule(lr){2-4}\cmidrule(lr){5-7}\cmidrule(lr){8-10} 
       Case & Contaminated  & Permuted      & Clean         & Contaminated  & Permuted      & Clean         & Contaminated  & Permuted      & Clean         \\ \midrule
Base  & 0.944 (0.005)          & 0.938 (0.005)          & 0.945 (0.005)          & 0.929 (0.028)          & 0.914 (0.020)          & 0.924 (0.018)          & 0.925 (0.007)          & 0.922 (0.009)          & 0.919 (0.008)          \\ 
No XS & 1.000 (0.000)          & 0.054 (0.112)          & 0.821 (0.057)          & 0.994 (0.006)          & 0.529 (0.220)          & 0.895 (0.030)          & 0.990 (0.002)          & 0.719 (0.009)          & 0.846 (0.017)          \\
RRR   & 0.510 (0.329)          & 0.523 (0.215)          & 0.528 (0.279)          & 0.985 (0.013)          & 0.536 (0.258)          & 0.836 (0.026)          & \textbf{0.943 (0.018)} & \textbf{0.927 (0.018)} & \textbf{0.926 (0.018)} \\
RBR   & 1.000 (0.001)          & 0.013 (0.017)          & 0.781 (0.013)          & 0.992 (0.006)          & 0.469 (0.166)          & 0.830 (0.029)          & 0.982 (0.011)          & 0.688 (0.243)          & 0.840 (0.046)          \\
TAP   & \textbf{0.933 (0.005)} & \textbf{0.935 (0.006)} & \textbf{0.936 (0.003)} & \textbf{0.926 (0.024)} & \textbf{0.911 (0.019)} & \textbf{0.906 (0.027)} & 0.912 (0.008)          & 0.907 (0.007)          & 0.910 (0.008) \\ \bottomrule
\end{tabular}}
\caption{Final-epoch F-score and standard deviation over five runs on PNEU. PNEU with stripe artifacts. Higher is better. Best permuted performance amongst XS methods in bold.}
\label{app:table:pneu_stripe}
\end{table*}

\begin{table*}[hb]
\centering
\resizebox{\textwidth}{!}{
\begin{tabular}{@{}lccccccccc@{}}
\toprule
       & \multicolumn{3}{c}{DenseNet-121}  & \multicolumn{3}{c}{ResNet-18}     & \multicolumn{3}{c}{VGG-16}        \\ \cmidrule(lr){2-4}\cmidrule(lr){5-7}\cmidrule(lr){8-10} 
       Case & Contaminated  & Permuted      & Clean         & Contaminated  & Permuted      & Clean         & Contaminated  & Permuted      & Clean         \\ \midrule
Base  & 0.658 (0.028)          & 0.623 (0.016)          & 0.609 (0.026)          & 0.634 (0.021)          & 0.630 (0.023)          & 0.607 (0.019)          & 0.493 (0.010)          & 0.505 (0.008)          & 0.485 (0.011)          \\ 
No XS & 0.015 (0.006)          & 1.515 (0.047)          & 1.485 (0.261)          & 0.021 (0.007)          & 1.484 (0.062)          & 1.803 (0.011)          & 0.010 (0.002)          & 1.475 (0.011)          & 1.710 (0.034)          \\
RRR   & 0.916 (0.202)          & 1.360 (0.055)          & 1.427 (0.060)          & 0.420 (0.037)          & 1.188 (0.023)          & 1.171 (0.063)          & 0.287 (0.056)          & 1.068 (0.020)          & 0.921 (0.052)          \\
RBR   & 0.373 (0.481)          & 1.424 (0.114)          & 1.669 (0.175)          & 0.053 (0.009)          & 1.392 (0.016)          & 1.615 (0.071)          & 0.111 (0.146)          & 1.364 (0.099)          & 1.628 (0.231)          \\
TAP   & \textbf{0.530 (0.194)} & \textbf{0.845 (0.112)} & \textbf{0.848 (0.149)} & \textbf{0.633 (0.048)} & \textbf{0.700 (0.025)} & \textbf{0.720 (0.053)} & \textbf{0.667 (0.011)} & \textbf{0.687 (0.009)} & \textbf{0.684 (0.008)} \\ \bottomrule
\end{tabular}}
\caption{Final-epoch mean absolute error and standard deviation over five runs on KNEE. KNEE with stripe artifacts. Lower is better. Best permuted performance amongst XS methods in bold.}
\label{app:table:knee_stripe}
\end{table*}

\clearpage
\section{Tabular Results: Ground-Truth Annotations}
\label{appx:gt_table}

We present tabular results on three versions of the test split: (i) the contaminated version $SC(X^*,y)$, (ii) the version with the permuted spurious signals $SC(X^*,permute(y))$ and (iii) a clean version $X^*$. A model that makes use of the spurious correlations from the contaminated data is expected to perform worse on the permuted and clean versions of the data. 

In Table~\ref{app:table:gt_mnist}, TAP is slightly outperformed by RRR and RBR on MNIST using the shallow two-layer CNN. TAP is specialised for CNNs. Shallower CNNs tend to have fully connected layers that make up a greater share of the model parameters than in deep CNNs. We conjecture that TAP may perform less well in such instances in comparison to RRR and RBR which are architecture-agnostic. We note, however, that TAP---while slightly worse---still matches RRR and RBR very closely.

\begin{table*}[ht]
\centering
\begin{tabular}{@{}lccc@{}}
\toprule
      Case & Contaminated  & Permuted      & Clean         \\ \midrule
Base  & 0.988 (0.001)          & 0.987 (0.001)          & 0.988 (0.001)          \\ 
No XS & 0.995 (0.001)          & 0.819 (0.006)          & 0.627 (0.017)          \\
RRR   & \textbf{0.987 (0.001)} & \textbf{0.987 (0.001)} & \textbf{0.987 (0.001)} \\
RBR   & 0.992 (0.001)          & 0.982 (0.001)          & 0.982 (0.004)          \\
TAP   & 0.983 (0.001)          & 0.979 (0.002)          & 0.979 (0.002)  \\ \bottomrule
\end{tabular}
\caption{Final-epoch accuracy and standard deviation over five runs on MNIST. Higher is better. Best permuted performance amongst XS methods in bold.}
\label{app:table:gt_mnist}
\end{table*}

\begin{table*}[ht]
\centering
\resizebox{\textwidth}{!}{
\begin{tabular}{@{}lccccccccc@{}}
\toprule
       & \multicolumn{3}{c}{DenseNet-121}  & \multicolumn{3}{c}{ResNet-18}     & \multicolumn{3}{c}{VGG-16}        \\ \cmidrule(lr){2-4}\cmidrule(lr){5-7}\cmidrule(lr){8-10} 
       Case & Contaminated  & Permuted      & Clean         & Contaminated  & Permuted      & Clean         & Contaminated  & Permuted      & Clean         \\ \midrule
Base  & 0.942 (0.009)          & 0.943 (0.010)          & 0.949 (0.006)          & 0.925 (0.018)          & 0.919 (0.015)          & 0.925 (0.016)          & 0.916 (0.012)          & 0.913 (0.011)          & 0.914 (0.009)          \\ 
No XS & 0.994 (0.004)          & 0.426 (0.287)          & 0.902 (0.037)          & 0.974 (0.020)          & 0.711 (0.175)          & 0.919 (0.029)          & 0.986 (0.003)          & 0.717 (0.018)          & 0.889 (0.009)          \\
RRR   & 0.994 (0.004)          & 0.429 (0.291)          & 0.900 (0.041)          & 0.922 (0.039)          & 0.891 (0.034)          & 0.914 (0.037)          & \textbf{0.936 (0.021)} & \textbf{0.928 (0.024)} & \textbf{0.932 (0.019)} \\
RBR   & 0.952 (0.009)          & 0.922 (0.017)          & 0.931 (0.018)          & \textbf{0.927 (0.022)} & \textbf{0.923 (0.025)} & \textbf{0.923 (0.023)} & 0.974 (0.006)          & 0.825 (0.007)          & 0.903 (0.009)          \\
TAP   & \textbf{0.941 (0.007)} & \textbf{0.945 (0.008)} & \textbf{0.950 (0.006)} & \textbf{0.926 (0.016)} & \textbf{0.923 (0.018)} & \textbf{0.918 (0.017)} & 0.930 (0.013)          & 0.922 (0.013)          & 0.928 (0.014) \\ \bottomrule
\end{tabular}}
\caption{Final-epoch F-score and standard deviation over five runs on PNEU. Higher is better. Best permuted performance amongst XS methods in bold.}
\end{table*}

\begin{table*}[ht]
\centering
\resizebox{\textwidth}{!}{
\begin{tabular}{@{}lccccccccc@{}}
\toprule
       & \multicolumn{3}{c}{DenseNet-121}  & \multicolumn{3}{c}{ResNet-18}     & \multicolumn{3}{c}{VGG-16}        \\ \cmidrule(lr){2-4}\cmidrule(lr){5-7}\cmidrule(lr){8-10} 
       Case & Contaminated  & Permuted      & Clean         & Contaminated  & Permuted      & Clean         & Contaminated  & Permuted      & Clean         \\ \midrule
Base  & 0.610 (0.026)          & 0.600 (0.019)          & 0.600 (0.011)          & 0.596 (0.018)          & 0.602 (0.015)          & 0.597 (0.018)          & 0.490 (0.015)          & 0.490 (0.017)          & 0.486 (0.015)          \\ 
No XS & 0.000 (0.000)          & 1.650 (0.000)          & 1.059 (0.078)          & 0.002 (0.001)          & 1.641 (0.018)          & 1.580 (0.809)          & 0.000 (0.001)          & 1.587 (0.017)          & 1.763 (0.085)          \\
RRR   & 1.179 (0.075)          & 1.233 (0.015)          & 1.241 (0.008)          & 0.819 (0.071)          & 1.224 (0.010)          & 1.238 (0.010)          & 0.275 (0.041)          & 1.167 (0.005)          & 1.243 (0.001)          \\
RBR   & 0.763 (0.355)          & 1.322 (0.131)          & 1.409 (0.249)          & 0.258 (0.188)          & 1.304 (0.028)          & 1.203 (0.033)          & 0.410 (0.556)          & 1.422 (0.204)          & 1.547 (0.326)          \\
TAP   & \textbf{0.597 (0.013)} & \textbf{0.623 (0.009)} & \textbf{0.619 (0.012)} & \textbf{0.595 (0.015)} & \textbf{0.607 (0.020)} & \textbf{0.613 (0.016)} & \textbf{0.399 (0.138)} & \textbf{0.587 (0.062)} & \textbf{0.585 (0.045)}  \\ \bottomrule
\end{tabular}}
\caption{Final-epoch mean absolute error and standard deviation over five runs on KNEE. Lower is better. Best permuted performance amongst XS methods in bold.}
\end{table*}

\clearpage
\section{Additional Experiments with Deep Architectures}
\label{appx:additional_cnns}

In this section, we provide further evidence for TAP's efficacy in Figure~\ref{fig:rebut} on two deep CNNs, EfficientNet-B0 and MobileNet-v2. Both are notably less vulnerable to spurious signals on PNEU, as the No XS models show little contamination sensitivity. Given limited time, we focused on the deeper EfficientNet-B0 and the KNEE dataset. We trained with TAP using a $\lambda$ of $10^{-6}$ on PNEU and $10^{-5}$ on KNEE. In addition, we used a $\lambda$ of $10^{-6}$ for RRR and $10^{3}$ for RBR on KNEE. 

RRR does not succeed in protecting the model from spurious signals. RBR matches Base model performance but incompletely reduces contamination sensitivity. TAP, however, succeeds in (i) matching Base model performance and (ii) eliminating contamination sensitivity. 

We note that RRR and RBR are not expected to \emph{non-spuriously} perform better than Base models, as Base models are trained on clean data. We tried to provide results for RRR and RBR for the other cases but ran out of time due to significantly slower training times for RRR and RBR. Tuning the 237-layer EfficientNet-B0 on KNEE took 50 GPU hours for RRR, 140 GPU hours for RBR and only 4 GPU hours for TAP. This difference in speed is consistent in order of magnitude with our findings for other deep CNNs (Table~\ref{tab:train_speed}). Nonetheless, these additional results show that TAP \emph{non-spuriously} matches clean Base model performance, which validates that TAP is effective for a wide range of CNNs. 

\begin{figure*}[ht]
\centering
\begin{subfigure}{0.45\columnwidth}
\includegraphics[width=0.99\columnwidth]{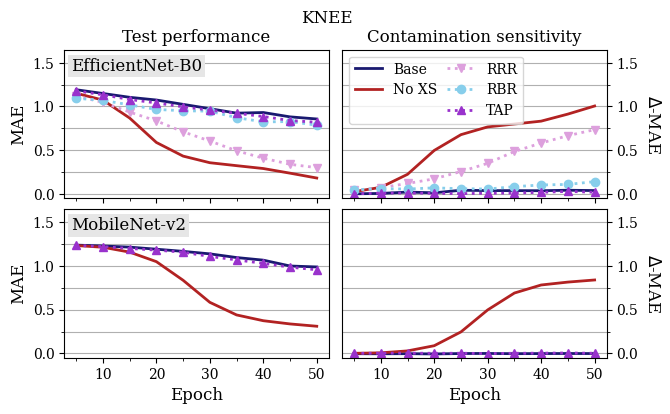}
\centering
\end{subfigure}%
\begin{subfigure}{0.45\columnwidth}
\centering
\includegraphics[width=0.99\columnwidth]{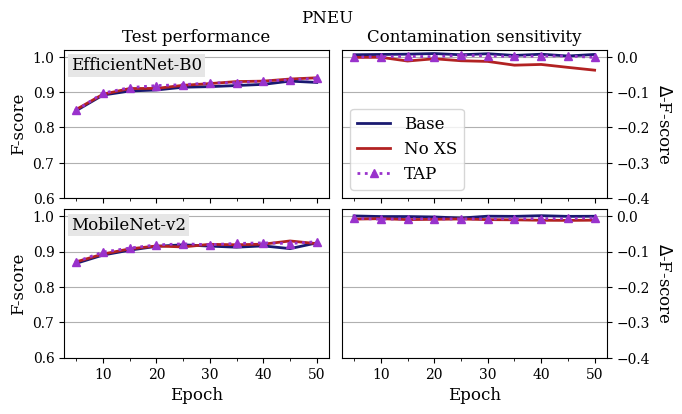}     
\end{subfigure}
\caption{Models trained with TAP on contaminated data match performance of models trained on clean data (Base) and show no contamination sensitivity. Results over five runs. Lower is better for KNEE, higher is better for PNEU.}
\label{fig:rebut}
\end{figure*}

\clearpage
\section{Teacher Annotations on MNIST}
\label{appx:teacher_mnist}

\begin{figure*}[h!]
\centering
\includegraphics[width=0.45\textwidth]{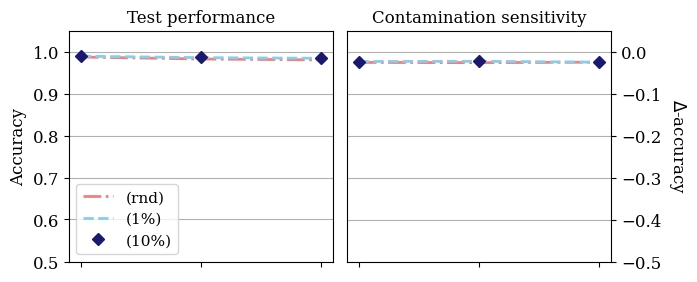} 
\caption{Test performance and contamination sensitivity on MNIST for TAP with random (rnd) and teacher (1\% or 10\%) annotations vs ground-truth. Higher is better. Performance and contamination sensitivity on MNIST do not appear to depend on the quality of annotations.}
\label{app:fig:mnist_ts}
\end{figure*}

\section{Tabular Results: Teacher Annotations}
\label{appx:teacher_table}

\begin{table*}[h]
\resizebox{\textwidth
}{!}{
\begin{tabular}{cccccccccc}
\toprule
     & \multicolumn{3}{c}{Contaminated}              & \multicolumn{3}{c}{Permuted}                  & \multicolumn{3}{c}{Clean}                     \\\cmidrule(lr){2-4}\cmidrule(lr){5-7}\cmidrule(lr){8-10} 
$\tau$  & 10\% teacher  & 1\% teacher   & Random        & 10\% teacher  & 1\% teacher   & Random        & 10\% teacher  & 1\% teacher   & Random        \\ \midrule
0.01 & 0.990 (0.001) & 0.989 (0.001) & 0.987 (0.001) & 0.966 (0.007) & 0.965 (0.006) & 0.961 (0.006) & 0.961 (0.009) & 0.961 (0.009) & 0.956 (0.008) \\
0.05 & 0.986 (0.001) & 0.985 (0.001) & 0.982 (0.000) & 0.963 (0.005) & 0.962 (0.006) & 0.956 (0.011) & 0.961 (0.008) & 0.961 (0.008) & 0.954 (0.012) \\
0.1  & 0.984 (0.001) & 0.984 (0.001) & 0.980 (0.001) & 0.960 (0.007) & 0.958 (0.009) & 0.955 (0.008) & 0.958 (0.009) & 0.957 (0.010) & 0.953 (0.010) \\ \bottomrule
\end{tabular}}
\caption{Final-epoch accuracy and standard deviation over five runs on MNIST using TAP with teacher annotations. Higher is better.}
\end{table*}

\begin{table}[h]
\resizebox{\textwidth}{!}{
\begin{tabular}{lcccccccccc}
\toprule
                              &      & \multicolumn{3}{c}{Contaminated}              & \multicolumn{3}{c}{Permuted}                  & \multicolumn{3}{c}{Clean}                     \\ \cmidrule(lr){3-5}\cmidrule(lr){6-8}\cmidrule(lr){9-11} 
                              Architecture & $\tau$  & 10\% teacher  & 1\% teacher   & Random        & 10\% teacher  & 1\% teacher   & Random        & 10\% teacher  & 1\% teacher   & Random        \\ \midrule
\multirow{3}{*}{Dense-121} & 0.01 & 0.952 (0.009) & 0.952 (0.008) & 0.994 (0.007) & 0.934 (0.011) & 0.935 (0.013) & 0.302 (0.349) & 0.949 (0.010) & 0.949 (0.009) & 0.890 (0.040) \\
                              & 0.05 & 0.941 (0.012) & 0.943 (0.011) & 0.996 (0.003) & 0.920 (0.019) & 0.918 (0.019) & 0.199 (0.293) & 0.934 (0.014) & 0.933 (0.014) & 0.886 (0.043) \\
                              & 0.1  & 0.959 (0.018) & 0.974 (0.010) & 0.994 (0.004) & 0.883 (0.022) & 0.858 (0.052) & 0.242 (0.299) & 0.925 (0.024) & 0.923 (0.027) & 0.892 (0.046) \\ \cmidrule(lr){1-11}
\multirow{3}{*}{ResNet-18}    & 0.01 & 0.922 (0.017) & 0.918 (0.016) & 0.963 (0.021) & 0.905 (0.018) & 0.898 (0.014) & 0.666 (0.305) & 0.918 (0.017) & 0.913 (0.015) & 0.884 (0.055) \\
                              & 0.05 & 0.910 (0.019) & 0.916 (0.017) & 0.948 (0.027) & 0.887 (0.023) & 0.877 (0.023) & 0.760 (0.149) & 0.897 (0.021) & 0.893 (0.022) & 0.872 (0.041) \\
                              & 0.1  & 0.924 (0.025) & 0.929 (0.019) & 0.945 (0.024) & 0.861 (0.028) & 0.854 (0.047) & 0.789 (0.112) & 0.886 (0.031) & 0.883 (0.031) & 0.872 (0.045) \\ \cmidrule(lr){1-11}
\multirow{3}{*}{VGG-16}       & 0.01 & 0.930 (0.006) & 0.930 (0.007) & 0.940 (0.008) & 0.906 (0.007) & 0.904 (0.008) & 0.875 (0.009) & 0.917 (0.006) & 0.916 (0.006) & 0.913 (0.006) \\
                              & 0.05 & 0.934 (0.007) & 0.935 (0.006) & 0.940 (0.007) & 0.880 (0.006) & 0.883 (0.010) & 0.872 (0.010) & 0.907 (0.004) & 0.909 (0.005) & 0.907 (0.005) \\
                              & 0.1  & 0.937 (0.008) & 0.937 (0.007) & 0.940 (0.007) & 0.872 (0.008) & 0.874 (0.011) & 0.868 (0.008) & 0.905 (0.005) & 0.907 (0.005) & 0.902 (0.005) \\ \bottomrule
\end{tabular}}
\caption{Final-epoch F-score and standard deviation over five runs on PNEU using TAP with teacher annotations. Higher is better.}
\end{table}

\begin{table}[h!]
\resizebox{\textwidth}{!}{
\begin{tabular}{lcccccccccc}
\toprule
                              &      & \multicolumn{3}{c}{Contaminated}              & \multicolumn{3}{c}{Permuted}                  & \multicolumn{3}{c}{Clean}                     \\ \cmidrule(lr){3-5}\cmidrule(lr){6-8}\cmidrule(lr){9-11} 
                              Architecture & $\tau$  & 10\% teacher  & 1\% teacher   & Random        & 10\% teacher  & 1\% teacher   & Random        & 10\% teacher  & 1\% teacher   & Random        \\ \midrule
\multirow{3}{*}{Dense-121} & 0.01 & 0.597 (0.075)                    & 0.575 (0.074)                   & 0.027 (0.040)              & 0.758 (0.043)                    & 0.776 (0.032)                   & 1.590 (0.099)              & 0.738 (0.028)                    & 0.746 (0.013)                   & 1.170 (0.204)              \\
                              & 0.05 & 0.618 (0.098)                    & 0.570 (0.049)                   & 0.099 (0.031)              & 0.917 (0.096)                    & 0.963 (0.038)                   & 1.377 (0.053)              & 0.866 (0.025)                    & 0.884 (0.031)                   & 1.273 (0.274)              \\
                              & 0.1  & 0.580 (0.110)                    & 0.563 (0.078)                   & 0.215 (0.131)              & 1.054 (0.098)                    & 1.058 (0.038)                   & 1.294 (0.049)              & 0.981 (0.136)                    & 0.976 (0.088)                   & 1.267 (0.268)              \\ \cmidrule(lr){1-11}
\multirow{3}{*}{ResNet-18}    & 0.01 & 1.000 (0.191)                    & 1.156 (0.196)                   & 0.869 (0.347)              & 1.009 (0.175)                    & 1.153 (0.194)                   & 1.148 (0.071)              & 1.023 (0.170)                    & 1.152 (0.193)                   & 1.112 (0.135)              \\
                              & 0.05 & 1.242 (0.015)                    & 1.249 (0.001)                   & 1.248 (0.002)              & 1.237 (0.012)                    & 1.245 (0.001)                   & 1.244 (0.001)              & 1.238 (0.012)                    & 1.245 (0.001)                   & 1.244 (0.002)              \\
                              & 0.1  & 1.249 (0.000)                    & 1.249 (0.000)                   & 1.249 (0.002)              & 1.245 (0.000)                    & 1.245 (0.000)                   & 1.245 (0.000)              & 1.245 (0.000)                    & 1.245 (0.000)                   & 1.245 (0.001)              \\ \cmidrule(lr){1-11}
\multirow{3}{*}{VGG-16}       & 0.01 & 0.198 (0.058)                    & 0.206 (0.053)                   & 0.066 (0.005)              & 0.889 (0.063)                    & 0.890 (0.071)                   & 1.212 (0.024)              & 0.860 (0.040)                    & 0.866 (0.056)                   & 1.168 (0.064)              \\
                              & 0.05 & 0.142 (0.027)                    & 0.136 (0.029)                   & 0.073 (0.008)              & 1.079 (0.029)                    & 1.091 (0.042)                   & 1.270 (0.024)              & 1.049 (0.084)                    & 1.070 (0.100)                   & 1.151 (0.098)              \\
                              & 0.1  & 0.106 (0.020)                    & 0.097 (0.018)                   & 0.077 (0.006)              & 1.154 (0.044)                    & 1.165 (0.041)                   & 1.283 (0.013)              & 1.079 (0.144)                    & 1.077 (0.149)                   & 1.305 (0.133)  \\ \bottomrule
\end{tabular}}
\caption{Final-epoch MAE and standard deviation over five runs on KNEE using TAP with teacher annotations. Lower is better.}
\end{table}

\clearpage
\section{Teacher Performance}
\label{appx:teacher_performance}

\begin{table*}[ht]
\centering
\begin{tabular}{lcccc}
\toprule
                Data (metric)       &  Teacher            & Contaminated  & Permuted      & Clean         \\ \midrule
\multirow{2}{*}{MNIST (accuracy)} & 10\% & 0.966 (0.002) & 0.965 (0.001) & 0.967 (0.001) \\ 
                       & 1\%  & 0.891 (0.010) & 0.888 (0.009) & 0.890 (0.010) \\ \cmidrule(lr){1-5}
\multirow{2}{*}{PNEU (F-score)}  & 10\% & 0.873 (0.014) & 0.869 (0.019) & 0.875 (0.014) \\ 
                       & 1\%  & 0.677 (0.159) & 0.668 (0.165) & 0.697 (0.130) \\ \cmidrule(lr){1-5}
\multirow{2}{*}{KNEE (MAE)}  & 10\% & 1.024 (0.053) & 1.026 (0.039) & 1.025 (0.032) \\
                       & 1\%  & 1.223 (0.037) & 1.226 (0.043) & 1.227 (0.038) \\ \bottomrule
\end{tabular}
\caption{Teacher performance of teachers trained on a 1\% or 10\% clean split and standard deviation over five runs. Higher is better for MNIST and PNEU, lower is better for KNEE.}
\end{table*}

\clearpage
\section{RRR and RBR with Teacher Annotations}
\label{appx:rrr_rbr_teacher}

For completeness, we report in this section additional results of using RRR and RBR with teacher annotations on PNEU and KNEE---we re-state results for TAP for ease of comparison. In general, the results suggest that RRR and RBR do not work as well with noisy teacher annotations as TAP, either because they fail to protect the model from relying on spurious signals or because of numerical instabilities during training. 

Table~\ref{app:table:pneu_ts_extra} shows for PNEU that RRR and RBR are not able to prevent models from learning spurious signals given the large drop in performance between contaminated and permuted performance. Table~\ref{app:table:knee_ts_extra} further support this finding on KNEE. The results also indicate that RRR and RBR may be more vulnerable to poor annotation quality, as several models trained with RRR and RBR and random annotations did not converge, especially at the higher threshold. In those cases, we report averages and standard deviations of the converged runs if at least three runs converged. 

\begin{table*}[h]
\resizebox{\textwidth}{!}{
\begin{tabular}{clccccccccc}
\toprule
                      &            & \multicolumn{3}{c}{DenseNet-121}                                         & \multicolumn{3}{c}{ResNet-18}                                            & \multicolumn{3}{c}{VGG-16}                                               \\ \cmidrule(lr){3-5} \cmidrule(lr){6-8} \cmidrule(lr){9-11}
$\tau$                & Case       & Contaminated           & Permuted               & Clean                  & Contaminated           & Permuted               & Clean                  & Contaminated           & Permuted               & Clean                  \\ \midrule
\multirow{6}{*}{0.01} & TAP (10\%) & \textbf{0.952 (0.009)} & \textbf{0.934 (0.011)} & \textbf{0.949 (0.010)} & \textbf{0.922 (0.017)} & \textbf{0.905 (0.018)} & \textbf{0.918 (0.017)} & \textbf{0.930 (0.006)} & \textbf{0.906 (0.007)} & \textbf{0.917 (0.006)} \\
                      & TAP (rnd)  & 0.994 (0.007)          & 0.302 (0.349)          & 0.890 (0.040)          & 0.963 (0.021)          & 0.666 (0.305)          & 0.884 (0.055)          & 0.940 (0.008)          & 0.875 (0.009)          & 0.913 (0.006)          \\
                      & RRR (10\%) & 0.994 (0.007)          & 0.348 (0.371)          & 0.904 (0.034)          & 0.955 (0.014)          & 0.805 (0.051)          & 0.888 (0.043)          & 0.492 (0.421)          & 0.496 (0.406)          & 0.496 (0.411)          \\
                      & RRR (rnd)  & 0.993 (0.005)          & 0.418 (0.305)          & 0.904 (0.027)          & 0.975 (0.017)          & 0.545 (0.126)          & 0.853 (0.052)          & 0.654 (0.300)          & 0.283 (0.303)          & 0.416 (0.379)          \\
                      & RBR (10\%) & 0.951 (0.031)          & 0.815 (0.132)          & 0.911 (0.021)          & 0.968 (0.020)          & 0.793 (0.097)          & 0.915 (0.038)          & 0.987 (0.003)          & 0.719 (0.021)          & 0.893 (0.009)          \\
                      & RBR (rnd)  & 0.973 (0.023)$\ddagger$          & 0.714 (0.233)$\ddagger$          & 0.904 (0.028)$\ddagger$          & 0.972 (0.011)          & 0.787 (0.054)          & 0.908 (0.044)          & 0.987 (0.003)          & 0.719 (0.021)          & 0.893 (0.009)          \\ \cmidrule(lr){1-11}
\multirow{6}{*}{0.1} & TAP (10\%) & \textbf{0.959 (0.018)} & \textbf{0.883 (0.022)} & \textbf{0.925 (0.024)} & \textbf{0.924 (0.025)} & \textbf{0.861 (0.028)} & \textbf{0.886 (0.031)} & \textbf{0.937 (0.008)} & \textbf{0.872 (0.008)} & \textbf{0.905 (0.005)} \\
                     & TAP (rnd)  & 0.994 (0.004)          & 0.242 (0.299)          & 0.892 (0.046)          & 0.945 (0.024)          & 0.789 (0.112)          & 0.872 (0.045)          & 0.940 (0.007)          & 0.868 (0.008)          & 0.902 (0.005)          \\
                     & RRR (10\%) & 0.993 (0.006)          & 0.378 (0.356)          & 0.909 (0.035)          & 0.980 (0.005)          & 0.738 (0.150)          & 0.911 (0.060)          & 0.953 (0.036)          & 0.392 (0.333)          & 0.823 (0.079)          \\
                     & RRR (rnd)  & 0.992 (0.007)          & 0.414 (0.336)          & 0.906 (0.033)          & 0.977 (0.008)          & 0.728 (0.156)          & 0.903 (0.059)          & 0.997 (0.003)          & 0.067 (0.102)          & 0.723 (0.153)          \\
                     & RBR (10\%) & \multicolumn{3}{c}{DNC}          & 0.969 (0.012)          & 0.804 (0.043)          & 0.908 (0.041)          & 0.987 (0.003)          & 0.718 (0.022)          & 0.892 (0.009)          \\
                     & RBR (rnd)  & \multicolumn{3}{c}{DNC}          & 0.961 (0.017)          & 0.808 (0.047)          & 0.902 (0.034)          & 0.987 (0.003)          & 0.717 (0.021)          & 0.892 (0.009)      \\ \bottomrule   
\end{tabular}}
\caption{Final-epoch F-score and standard deviation over five runs on PNEU using XS with teacher (10\%) and random (rnd) annotations. Cases that did not converge (DNC) at least three times are marked as such. Cases with $\ddagger$ twice DNC. Lower is better. Best permuted performance per architecture and threshold in bold.}
\label{app:table:pneu_ts_extra}
\end{table*}

\begin{table*}[h]
\resizebox{\textwidth}{!}{
\begin{tabular}{clccccccccc}
\toprule
                      &            & \multicolumn{3}{c}{DenseNet-121}                                         & \multicolumn{3}{c}{ResNet-18}                                            & \multicolumn{3}{c}{VGG-16}                                               \\ \cmidrule(lr){3-5} \cmidrule(lr){6-8} \cmidrule(lr){9-11}
$\tau$                & Case       & Contaminated           & Permuted               & Clean                  & Contaminated           & Permuted               & Clean                  & Contaminated           & Permuted               & Clean                  \\ \midrule
\multirow{6}{*}{0.01}    & TAP (10\%) & \textbf{0.597 (0.075)} & \textbf{0.758 (0.043)} & \textbf{0.738 (0.028)} & \textbf{1.000 (0.191)} & \textbf{1.009 (0.175)} & \textbf{1.023 (0.170)} & \textbf{0.198 (0.058)} & \textbf{0.889 (0.063)} & \textbf{0.860 (0.040)} \\
                     & TAP (rnd)  & 0.027 (0.040)          & 1.590 (0.099)          & 1.170 (0.204)          & 0.869 (0.347)          & 1.148 (0.071)          & 1.112 (0.135)          & 0.066 (0.005)          & 1.212 (0.024)          & 1.168 (0.064)          \\
                     & RRR (10\%) & 1.142 (0.149)          & 1.232 (0.021)          & 1.254 (0.004)          & 1.202 (0.025)          & 1.238 (0.004)          & 1.259 (0.002)          & 0.093 (0.006)          & 1.375 (0.020)          & 1.259 (0.001)          \\
                     & RRR (rnd)  & 1.164 (0.107)$\ddagger$          & 1.244 (0.017)$\ddagger$          & 1.261 (0.009)$\ddagger$          & 1.156 (0.174)$\dagger$          & 1.246 (0.002)$\dagger$          & 1.259 (0.001)$\dagger$          & 0.085 (0.003)          & 1.429 (0.014)          & 1.253 (0.006)          \\
                     & RBR (10\%) & 0.494 (0.428)          & 1.568 (0.246)          & 1.254 (0.002)          & 0.047 (0.037)          & 1.512 (0.132)          & 1.197 (0.110)          & 0.889 (0.197)$\dagger$          & 1.203 (0.035)$\dagger$          & 1.259 (0.000)$\dagger$          \\
                     & RBR (rnd)  & 0.011 (0.003)$\dagger$          & 1.650 (0.002)$\dagger$          & 1.174 (0.109)$\dagger$          & 0.950 (0.311)          & 1.176 (0.138)          & 1.195 (0.183)          & \multicolumn{3}{c}{DNC}          \\ \cmidrule(lr){1-11}
\multirow{6}{*}{} & TAP (10\%) & \textbf{0.580 (0.110)} & \textbf{1.054 (0.098)} & \textbf{0.981 (0.136)} & 1.249 (0.000)          & 1.245 (0.000)          & 1.245 (0.000)          & \textbf{0.106 (0.020)} & \textbf{1.154 (0.044)} & \textbf{1.079 (0.144)} \\
                  & TAP (rnd)  & 0.215 (0.131)          & 1.294 (0.049)          & 1.267 (0.268)          & 1.249 (0.002)          & 1.245 (0.000)          & 1.245 (0.001)          & 0.077 (0.006)          & 1.283 (0.013)          & 1.305 (0.133)          \\
                  & RRR (10\%) & 0.768 (0.229)          & 1.251 (0.034)          & 1.249 (0.010)          & \textbf{0.723 (0.078)} & \textbf{1.230 (0.015)} & \textbf{1.251 (0.005)} & 0.073 (0.012)          & 1.458 (0.024)          & 1.257 (0.001)          \\
                  & RRR (rnd)  & 1.340 (0.681)$\ddagger$          & 1.592 (0.590)$\ddagger$          & 1.655 (0.688)$\ddagger$          & 0.650 (0.108)          & 1.259 (0.021)          & 1.250 (0.009)          & 0.034 (0.010)          & 1.557 (0.027)          & 1.586 (0.212)          \\
                  & RBR (10\%) & 0.578 (0.830)$\dagger$          & 1.585 (0.256)$\dagger$          & 1.409 (0.434)$\dagger$          & \multicolumn{3}{c}{DNC}                                                  & \multicolumn{3}{c}{DNC}                                                  \\
                  & RBR (rnd)  & \multicolumn{3}{c}{DNC}                                                  & \multicolumn{3}{c}{DNC}                                                  & \multicolumn{3}{c}{DNC}        \\ \bottomrule   
\end{tabular}}
\caption{Final-epoch MAE and standard deviation over five runs on KNEE using XS with teacher (10\%) and random (rnd) annotations. Cases that did not converge (DNC) at least three times are marked as such. Cases with $\dagger$ once DNC, cases with $\ddagger$ twice DNC. Lower is better. Best permuted performance per architecture and threshold in bold.}
\label{app:table:knee_ts_extra}
\end{table*}

\clearpage
\section{Comparison with Knowledge Distillation Methods}
\label{appx:kd_comparison}

In this section we provide a comparison of TAP with knowledge distillation methods. We compare our results from TAP with teacher annotations ($\tau = 0.01$) with attention transfer (AT) \citep{zagoruykoPayingMoreAttention2017}, Jacobian transfer (JT) \citep{srinivasKnowledgeTransferJacobian2018} and vanilla knowledge distillation (KD) \citep{hintonDistillingKnowledgeNeural2015}. Note that while these methods all feature teacher-student setups, our method differs in that the teacher model does not directly provide targets for the student to match but is instead used to identify areas to be ignored by the student. AT transfers the activations of the teacher model to the student. While our method has a similar focus on student activations, teachers do not transfer activations but instead provide input-level explanations. JT, on the other hand, does focus on input-level explanations in the form of Jacobian but again makes the student match them. Moreover, as JT requires calculating second-order derivatives, training times and memory usage suffer. KD, lastly, focuses only on logits and therefore leaves the ``reasons'' of the student unconstrained.  

We use the same ResNet-18 teachers from our main experiments and tune $\lambda$ as before but expand the range to $10^{-9}$ to $10^{4}$ for AT, JT and KD---we report chosen values in Table~\ref{app:table:kd_paras}. 

On MNIST (Table~\ref{app:table:kd_mnist}), both AT and JT are not successful at protecting the model from learning spurious correlations. Surprisingly, KD is able to yield a good performance closely matching that of TAP for the 10\%-teacher. The performance of KD reduces slightly, however, when we switch to the less well-trained 1\%-teacher. TAP, on the other hand, consistently performs well with both the less and better trained teachers. On PNEU (Table~\ref{app:table:kd_pneu}), JT and KD do not mitigate contamination sensitivity well, given the difference in contaminated and permuted performances. AT, in comparison, performs better and with ResNet-18 and VGG-16 appears to perform similarly to TAP. AT, however, falls behind TAP with DenseNet-121. On KNEE (Table~\ref{app:table:kd_knee}), JT and KD again do not perform well. With ResNet-18, AT yields the best results on the permuted data but is characterised by contamination sensitivity which is absent under TAP. TAP, otherwise, produces the best results on the permuted data and is the most successful of the methods in mitigating contamination sensitivity. 

Overall, JT and KD do not seem to be very successful at transferring contamination insensitivity from the teacher to the student in the context of deeper CNNs. While AT is more successful and in some cases performs on par with TAP, its success varies more with the architecture. In addition, AT requires that the number of activations is aligned between teacher and student which can be challenging when the models vary significantly in depth. TAP, on the other hand, makes use of input-level annotations which can be downsampled for any number of activations.


\begin{table*}[h]
\centering
\begin{tabular}{clccc}
\toprule
Teacher               & Loss & Contaminated  & Permuted      & Clean         \\ \midrule
\multirow{4}{*}{10\%} & TAP & 0.990 (0.001)          & 0.966 (0.007)          & 0.961 (0.009)          \\
                      & AT  & 0.995 (0.000)          & 0.851 (0.010)          & 0.689 (0.035)          \\
                      & JT  & 0.996 (0.000)          & 0.820 (0.007)          & 0.640 (0.025)          \\
                      & KD  & \textbf{0.984 (0.000)} & \textbf{0.968 (0.000)} & \textbf{0.970 (0.003)} \\ \cmidrule(lr){1-5}
\multirow{4}{*}{1\%}  & TAP & \textbf{0.989 (0.001)} & \textbf{0.965 (0.006)} & \textbf{0.961 (0.009)} \\
                      & AT  & 0.994 (0.000)          & 0.822 (0.015)          & 0.642 (0.033)          \\
                      & JT  & 0.996 (0.000)          & 0.820 (0.007)          & 0.640 (0.025)          \\
                      & KD  & 0.966 (0.002)          & 0.929 (0.006)          & 0.922 (0.009) \\ \bottomrule
\end{tabular}
\caption{Final-epoch accuracy and standard deviation over five runs on MNIST. Higher is better. Best permuted performance in bold.}
\label{app:table:kd_mnist}
\end{table*}

\begin{table*}[h]
\centering
\resizebox{\textwidth}{!}{
\begin{tabular}{clccccccccc}
\toprule
\multicolumn{1}{l}{}  & \multicolumn{1}{l}{} & \multicolumn{3}{c}{DenseNet-121}              & \multicolumn{3}{c}{ResNet-18}                 & \multicolumn{3}{c}{VGG-16}                            \\ \cmidrule(lr){3-5}\cmidrule(lr){6-8}\cmidrule(lr){9-11}
Teacher               & Loss                 & Contaminated  & Permuted      & Clean         & Contaminated  & Permuted      & Clean         & Contaminated  & Permuted      & Clean                 \\ \midrule
\multirow{4}{*}{10\%} & TAP & \textbf{0.952 (0.009)} & \textbf{0.934 (0.011)} & \textbf{0.949 (0.010)} & \textbf{0.922 (0.017)} & \textbf{0.905 (0.018)} & \textbf{0.918 (0.017)} & 0.930 (0.006)          & 0.906 (0.007)          & 0.917 (0.006)          \\
                      & AT  & 0.970 (0.006)          & 0.892 (0.024)          & 0.948 (0.006)          & 0.922 (0.022)          & 0.903 (0.020)          & 0.918 (0.020)          & \textbf{0.937 (0.020)} & \textbf{0.909 (0.018)} & \textbf{0.926 (0.020)} \\
                      & JT  & 0.992 (0.006)          & 0.421 (0.303)          & 0.901 (0.021)          & 0.980 (0.009)          & 0.734 (0.123)          & 0.923 (0.043)          & 0.987 (0.003)          & 0.719 (0.021)          & 0.893 (0.009)          \\
                      & KD  & 0.898 (0.012)          & 0.892 (0.014)          & 0.899 (0.011)          & 0.985 (0.003)          & 0.718 (0.112)          & 0.924 (0.043)          & 0.915 (0.016)          & 0.898 (0.011)          & 0.915 (0.012)          \\ \cmidrule(lr){1-11}
\multirow{4}{*}{1\%}  & TAP & \textbf{0.952 (0.008)} & \textbf{0.935 (0.013)} & \textbf{0.949 (0.009)} & 0.918 (0.016)          & 0.898 (0.014)          & 0.913 (0.015)          & \textbf{0.930 (0.007)} & \textbf{0.904 (0.008)} & \textbf{0.916 (0.006)} \\
                      & AT  & 0.972 (0.007)          & 0.902 (0.008)          & 0.947 (0.008)          & \textbf{0.923 (0.015)} & \textbf{0.904 (0.018)} & \textbf{0.919 (0.014)} & 0.920 (0.036)          & 0.895 (0.041)          & 0.908 (0.038)          \\
                      & JT  & 0.994 (0.005)          & 0.375 (0.343)          & 0.907 (0.031)          & 0.981 (0.008)          & 0.731 (0.119)          & 0.925 (0.043)          & 0.987 (0.003)          & 0.719 (0.021)          & 0.893 (0.009)          \\
                      & KD  & 0.706 (0.173)          & 0.695 (0.175)          & 0.726 (0.149)          & 0.985 (0.003)          & 0.716 (0.112)          & 0.922 (0.042)          & 0.755 (0.127)          & 0.725 (0.142)          & 0.732 (0.149) \\ \bottomrule
\end{tabular}}
\caption{Final-epoch F-score and standard deviation over five runs on PNEU. Higher is better. Best permuted performance in bold.}
\label{app:table:kd_pneu}
\end{table*}

\clearpage

\begin{table*}[h]
\centering
\resizebox{\textwidth}{!}{
\begin{tabular}{clccccccccc}
\toprule
\multicolumn{1}{l}{}  & \multicolumn{1}{l}{} & \multicolumn{3}{c}{DenseNet-121}              & \multicolumn{3}{c}{ResNet-18}                 & \multicolumn{3}{c}{VGG-16}                            \\ \cmidrule(lr){3-5}\cmidrule(lr){6-8}\cmidrule(lr){9-11}
Teacher               & Loss                 & Contaminated  & Permuted      & Clean         & Contaminated  & Permuted      & Clean         & Contaminated  & Permuted      & Clean                 \\ \midrule
\multirow{4}{*}{10\%} & TAP & \textbf{0.627 (0.041)} & \textbf{0.745 (0.037)} & \textbf{0.733 (0.030)} & 1.000 (0.191)          & 1.009 (0.175)          & 1.023 (0.170)          & \textbf{0.198 (0.058)} & \textbf{0.889 (0.063)} & \textbf{0.860 (0.040)} \\
                      & AT  & 0.455 (0.079)          & 1.005 (0.035)          & 0.930 (0.028)          & \textbf{0.528 (0.055)} & \textbf{0.882 (0.021)} & \textbf{0.859 (0.031)} & 0.116 (0.046)          & 1.148 (0.050)          & 0.949 (0.035)          \\
                      & JT  & 1.779 (0.641)          & 1.780 (0.642)          & 1.781 (0.633)          & 1.963 (0.676)          & 1.965 (0.680)          & 1.963 (0.668)          & 2.150 (0.822)          & 2.151 (0.827)          & 2.148 (0.812)          \\
                      & KD  & 0.996 (0.093)          & 1.038 (0.080)          & 1.010 (0.056)          & 0.814 (0.079)          & 1.039 (0.048)          & 1.038 (0.039)          & 0.985 (0.078)          & 1.022 (0.064)          & 1.066 (0.066)          \\ \cmidrule(lr){1-11}
\multirow{4}{*}{1\%}  & TAP & \textbf{0.575 (0.074)} & \textbf{0.776 (0.032)} & \textbf{0.746 (0.013)} & 1.156 (0.196)          & 1.153 (0.194)          & 1.152 (0.193)          & \textbf{0.206 (0.053)} & \textbf{0.890 (0.071)} & \textbf{0.866 (0.056)} \\
                      & AT  & 0.466 (0.106)          & 1.037 (0.044)          & 0.959 (0.042)          & \textbf{0.548 (0.044)} & \textbf{0.927 (0.025)} & \textbf{0.924 (0.040)} & 0.126 (0.041)          & 1.177 (0.014)          & 0.954 (0.039)          \\
                      & JT  & 1.608 (0.643)          & 1.607 (0.647)          & 1.612 (0.636)          & 2.196 (0.742)          & 2.198 (0.746)          & 2.192 (0.734)          & 2.150 (0.822)          & 2.153 (0.826)          & 2.148 (0.812)          \\
                      & KD  & 1.230 (0.021)          & 1.223 (0.018)          & 1.235 (0.027)          & 1.065 (0.093)          & 1.220 (0.037)          & 1.230 (0.040)          & 1.228 (0.019)          & 1.233 (0.019)          & 1.249 (0.008) \\ \bottomrule
\end{tabular}}
\caption{Final-epoch MAE and standard deviation over five runs on KNEE. Lower is better. Best permuted performance in bold.}
\label{app:table:kd_knee}
\end{table*}

\begin{table*}[h]
\centering
\begin{tabular}{lcccccccccccc}
\toprule
       & \multicolumn{3}{c}{DenseNet-121}  & \multicolumn{3}{c}{ResNet-18}     & \multicolumn{3}{c}{VGG-16}   & \multicolumn{3}{c}{Two-layer CNN}         \\ \cmidrule(lr){2-4}\cmidrule(lr){5-7}\cmidrule(lr){8-10}\cmidrule(lr){11-13} 
     Data & AT       & JT       & KD      & AT       & JT       & KD     & AT       & JT       & KD      & AT       & JT       & KD \\ \midrule
PNEU  & $10^{2}$ & $10^{-9}$ & $10^{-4}$ & $10^{0}$ & $10^{-3}$ & $10^{4}$ & $10^{3}$ & $10^{-6}$ & $10^{2}$ & -     & -     & -     \\
KNEE  & $10^{3}$ & $10^{1}$ & $10^{1}$ & $10^{3}$ & $10^{1}$ & $10^{3}$ & $10^{3}$ & $10^{1}$ & $10^{2}$ & -     & -     & -     \\
MNIST & -     & -     & -     & -     & -     & -     & -     & -     & -     & $10^{0}$ & $10^{-7}$ & $10^{0}$ \\ \bottomrule
\end{tabular}
\caption{Choices for $\lambda$ for all knowledge distillation methods.}
\label{app:table:kd_paras}
\end{table*}

\end{document}